\documentclass[twoside]{article}

\usepackage[utf8]{inputenc} % allow utf-8 input
\usepackage[T1]{fontenc}    % use 8-bit T1 fonts
\usepackage{hyperref}       % hyperlinks
\usepackage{url}            % simple URL typesetting
\usepackage{booktabs}       % professional-quality tables
\usepackage{amsfonts}       % blackboard math symbols
\usepackage{nicefrac}       % compact symbols for 1/2, etc.
\usepackage{microtype}      % microtypography

\newcommand{\arxiv}{true}

\usepackage[active]{srcltx}

\usepackage{xcolor}
\usepackage{amsthm,amsmath,amsfonts,amssymb}
\allowdisplaybreaks
\usepackage{tikz}
\usetikzlibrary{shapes,arrows,backgrounds,positioning,plotmarks,calc,patterns,plothandlers,decorations.pathmorphing,decorations.pathreplacing,cd}
\usepackage{todonotes}
\usepackage{multirow}
\usepackage{enumitem}
\usepackage{wrapfig}

\usepackage{autobreak}

\usepackage{etoolbox}

\usepackage{maplestd2e}
\DefineParaStyle{Maple Bullet Item}
\DefineParaStyle{Maple Heading 1}
\DefineParaStyle{Maple Warning}
\DefineParaStyle{Maple Heading 4}
\DefineParaStyle{Maple Heading 2}
\DefineParaStyle{Maple Heading 3}
\DefineParaStyle{Maple Dash Item}
\DefineParaStyle{Maple Error}
\DefineParaStyle{Maple Title}
\DefineParaStyle{Maple Text Output}
\DefineParaStyle{Maple Normal}
\DefineCharStyle{Maple 2D Output}
\DefineCharStyle{Maple 2D Input}
\DefineCharStyle{Maple Maple Input}
\DefineCharStyle{Maple 2D Math}
\DefineCharStyle{Maple Hyperlink}

\newcommand{\Z}{{\mathbb{Z}}}

\newcommand{\R}{{\mathbb{R}}}
\newcommand{\F}{{\mathcal{F}}}
\newcommand{\G}{{\mathcal{G}}}

\newcommand{\V}{{\mathcal{V}}}

\DeclareMathOperator{\GP}{\mathcal{GP}}
\DeclareMathOperator{\sol}{sol}

\newcommand{\xleftarrowdbl}[2][]{%
	\leftarrow\mathrel{\mkern-14mu}\xleftarrow[#1]{#2}
}
\newtheorem{theorem}{Theorem}[section]
\newtheorem{lemma}[theorem]{Lemma}
\newtheorem{proposition}[theorem]{Proposition}
\newtheorem{corollary}[theorem]{Corollary}
\theoremstyle{definition}

\newtheorem{example}[theorem]{Example}
\newtheorem{remark}[theorem]{Remark}

\usepackage{microtype}
\usepackage{graphicx}
\usepackage{subfigure}
\usepackage{booktabs} % for professional tables
\usepackage{enumitem}

\usepackage{hyperref}

\usepackage[accepted]{aistats2021}
\usepackage[round]{natbib}

\title{Linearly Constrained Gaussian Processes with Boundary Conditions}

\author{Markus Lange-Hegermann\\
	Department of Electrical Engineering and Computer Science\\
	OWL University of Applied Sciences and Arts\\
	\texttt{markus.lange-hegermann@th-owl.de }}

\date{}

\begin{document}



\twocolumn

\maketitle

\begin{abstract}
	One goal in Bayesian machine learning is to encode prior knowledge into prior distributions, to model data efficiently.
	We consider prior knowledge from  systems of linear partial differential equations together with their boundary conditions.
	We construct multi-output Gaussian process priors with realizations in the solution set of such systems, in particular only such solutions can be represented by Gaussian process regression.
	The construction is fully algorithmic via Gr\"obner bases and it does not employ any approximation.
	It builds these priors combining two parametrizations via a pullback: the first parametrizes the solutions for the system of differential equations and the second parametrizes all functions adhering to the boundary conditions.
\end{abstract}

\section{Introduction}

Gaussian processes \citep{RW} are very data efficient.
Hence, they are the prime regression technique for small datasets and applied when data is rare or expensive to produce.
Applications range from robotics \citep{lima-lars-18}, biology \citep{LawrencePNAS2015}, global optimization \citep{OGR09}, anomaly detection \cite{Berns20}, hyperparameter search \cite{AutoWEKA}
%astrophysics \citep{GHSQuasars} 
to engineering \citep{TLRB_gp}.
A Gaussian process can be viewed as a suitable probability distribution on a set of functions, which we can condition on observations using Bayes' rule.
This avoids overfitting.
Due to the self-conjugacy of the Gaussian distribution, the posterior is again Gaussian.
The mean function of the posterior is used for regression and the variance quantifies uncertainty.
%Additionally, one can also easily sample random functions, so-called realizations, from this distribution.
For a suitable covariance function of the prior, the posterior can approximate any behavior present in data, even in noisy or unstructured data.
%In that sense, they are a robust regression technique, which can easily and automatically be applied without oversight.

Any prior knowledge about the regression problem should be incorporated into the prior.
Then, the preciously rare measurement data can be used to refine and improve on this prior knowledge, instead of needing to relearn it.
The prior knowledge is usually encoded into the covariance structure of the Gaussian process, cf.\ \citep[\S4]{RW} or \citep{duvenaud-thesis-2014}.
Gaussian process regression differs in philosophy from deep learning, where the latter thrives on extracting knowledge from a lot of data but struggles with one-shot learning and encoding prior knowledge, which is usually done via pretraining on similar data.

Prior knowledge is often given by physical laws.
In particular, it is important to include linear differential equations into machine learning frameworks.
Gaussian processes that adhere to such a set of linear differential equations were constructed several times in the literature \citep{NoisyLinearOperatorEquationsGP,MacedoCastro2008,sarkka2011linear,scheuerer2012covariance,Wahlstrom13modelingmagnetic,MagneticFieldGP,LinearlyConstrainedGP,raissi2017machine,raissi2018hidden,jidling2018probabilistic}.
All realizations and the mean function of the posterior strictly 
\footnote{%
	For notational simplicity, we refrain from using the phrase ``almost surely'' in this paper, e.g.\ by assuming separability.
}
 satisfy these physical laws.
Such Gaussian processes exist if and only if the set of linear differential equations describes a controllable system\footnote{%
	Controllable systems have ``big'' sets of solutions, even after adding boundary conditions.
	As there is no unique solution, one does regression or control in those systems, instead of (numerically) solving them as is usually done for systems with ``small'' solution sets.
}.
Their construction can be completely automatized by symbolic algorithms from algebraic system theory, which again strongly build on Gr\"obner bases \citep{LH_AlgorithmicLinearlyConstrainedGaussianProcesses}.

While the above approaches are exact, there are also approximate approaches to include partial differential equations in Gaussian process and more generally machine learning.
For example, various forms of posterior regularization \citep{ganchev2010posterior,SongKernel2016,yuan2020macroscopic} can flexibly consider any differential equation.
The paper \cite{numericalGPs} constructed Gaussian processes on numerical difference approximation schemes of differential equations.
Gaussian processes have been used to estimate conservation laws (\citealp{LDEGaussianProcesses,FunctionalGPsPDE,nguyen2016gaussian}).
In \citep{yang2018physics}, a Gaussian process prior is approximated from an MCMC scheme build on numerical simulations.

Usually, differential equations come with boundary conditions.
Hence, a description of boundary conditions in a machine learning framework is highly desirable.
In the special case of ODEs, boundary conditions behave as data points, hence one only needs finite-dimensional data to specify them.
These data points can be trivially included into a Gaussian process %Example~\ref{example_ode_boundary};
	\citep{kocijan2004gaussian,CalderheadAccelerating2009,barber2014gaussian,GOODE} and other machine learning methods \citep{chen2018neural,Raissi2018Deep,sarkka2019applied}.
This paper claims no originality for ODEs.

For boundary conditions of PDEs, one would need functions (specified by infinite dimensional data) to describe the boundary conditions.
Solving this problem \emph{exactly} and without any approximation is the main contribution of this paper: the construction of (non-stationary) Gaussian process priors combining differential equations and general boundary conditions.
This construction is again based on symbolically building parametrizations using Gr\"obner bases, as in \citep{LH_AlgorithmicLinearlyConstrainedGaussianProcesses}.

More precisely, given a system of linear differential equations with rational (or, as a special case, constant) coefficient of a controllable system defined by an operator matrix and boundary conditions defined by the zero set of a polynomial ideal, we construct a Gaussian process prior of the corresponding set of smooth solutions.
In particular, a regression model constructed from this Gaussian process prior has \emph{only} solutions of the system as realizations.
We need no approximations.

Using the results of this paper, one can add information to Gaussian processes by
\begin{enumerate}[label=(\roman*)]
	\item conditioning on data points (Bayes' rule), 
	\item restricting to solutions of linear operator matrices \cite{LH_AlgorithmicLinearlyConstrainedGaussianProcesses}, and
	\item adding boundary conditions (this paper).
\end{enumerate}
Since these constructions are compatible, we can combine \emph{strict, global information} from equations and boundary conditions with \emph{noisy, local information} from observations.
This paper is an example in how symbolic techniques can help data driven machine learning.
All results are mathematically proven in the appendices and the algorithms are demonstrated on toy examples with only one or two data points, an extreme form of one-shot learning.
The code for reproduction of the results is given in \ifthenelse{\equal{\arxiv}{true}}{Appendix~\ref{appendix_code}}{the appendix} and the (very small amount of) data is completely given in the text of this paper. %Appendix~\ref{appendix_code}.

The novelty in this paper does not lie in either of its techniques, which are well-known either in algebraic system theory or machine learning.
Rather, this paper combines these techniques and thereby presents a novel framework to deal with learning from data in the presence of linear controllable \emph{partial} differential equations and boundary conditions.
We found it hard to compare to the state of the art, as there currently is no comparable technique, except the superficially similar paper \citep{NoisyLinearOperatorEquationsGP} discussed in Remark~\ref{remark_graepl} and a plethora of machine learning techniques designed for \emph{ordinary} differential equations.
The only exception is \cite{gulian2020gaussian}, which considers inhomogeneous linear differential equations with boundary conditions using the spectral decomposition of a covariance function \citep{solin2019know}, where the right hand side is specified approximately by data.
These approaches allow to approximately specify a prior for the solution of the differential equation, instead of specifying the prior for the parametrizing function as in this paper.

We recall Gaussian processes and their connection to linear operators in Section~\ref{section_GP} and summarize the construction of Gaussian processes adhering to linear operators in Section~\ref{section_GPs_solutions}.
Describing boundary conditions as parametrizations is surprisingly simple (Section~\ref{section_boundary}).
Theorem~\ref{theorem_combining_parametrizations} describes the core construction of this paper, which allows to check whether and how two parametrizations are combinable.
In Section~\ref{section_inhomogeneous} we construct boundary conditions with non-zero right hand sides using the fundamental theorem on homomorphisms.

\section{Operators and Gaussian Processes}\label{section_GP}

%This section gives a short exposition of Gaussian processes and their connection to linear operators.
%\subsection{Gaussian processes}\label{subsection_GP}
% Gaussian processes excel at non-parametric regression.
% % Usually, one places a Gaussian process prior on an unknown behavior and then uses Bayes' theorem to condition on observations.
% Each measurement point gives an additional term\footnote{For most covariance functions, these terms do not cancel.} in the regression function.
% % Thus, arbitrarily complicated functions can be approximated, given enough observations.
% % Usually, surprisingly few observations suffice for a good model.
A \emph{Gaussian process} $g=\GP(\mu,k)$ is a probability distribution on the evaluations of functions $\R^d\to\R^\ell$ such that function values $g(x_1),\ldots,g(x_n)$
% at $x_1,\ldots,x_n\in\R^d$
are jointly Gaussian.
It is specified by a \emph{mean function}
$
  \mu:\R^d\to\R^\ell:x\mapsto E(g(x))
$
and a positive semidefinite \emph{covariance function}
\begin{align*}
    k: \R^d\times\R^d &\to \R^{\ell\times\ell}_{\succeq0}: \\
    (x,x') &\mapsto E\left((g(x)-\mu(x))(g(x')-\mu(x'))^T\right)\mbox{ .}
\end{align*}
All higher moments exists and are uniquely determined by $\mu$ and $k$, all higher cumulants are zero.
We often restrict the domain of a Gaussian process to a subset of $\R^d$.

Assume the regression model $y_i=g(x_i)$ and condition on observations 
$(x_i,y_i)\in\R^{1\times d}\times\R^{1\times \ell}$ for $i=1,\ldots,n$.
Denote by $k(x,X)\in \R^{\ell\times\ell n}$ resp.\ $k(X,X)\in \R^{\ell n\times\ell n}_{\succeq0}$ the (covariance) matrices obtained by concatenating the matrices $k(x,x_j)$ resp.\ the positive semidefinite block partitioned matrix with blocks $k(x_i,x_j)$.
Write $\mu(X)$ resp.\ $y\in \R^{1\times \ell n}$ for the row vector obtained by concatenating the rows $\mu(x_i)$ resp.\ $y_i$.
The posterior is the Gaussian process
%\begin{align*}
%  \GP\Big( \quad x\mapsto &\ \mu(x)+(y-\mu(X))k(X,X)^{-1}k(x,X)^T,\\
%                        (x,x')\mapsto &\ k(x,x')-k(x,X)k(X,X)^{-1}k(x',X)^T\Big)\mbox{.}
%\end{align*}
\begin{align*}
	\GP\Big(& \mu(x)+(y-\mu(X))k(X,X)^{-1}k(x,X)^T,\\ &k(x,x')-k(x,X)k(X,X)^{-1}k(x',X)^T\Big)\mbox{.}
\end{align*}
Its mean function can be used as regression model and its variance as model uncertainty.

%\subsection{Linear operators on Gaussian processes}\label{subsection_operators_GPs}

Gaussian processes are the linear objects among stochastic processes and their rich connection with linear operators is present everywhere in this paper.
In particular, the class of Gaussian processes is closed under linear operators once mild assumptions hold.
Now, we formalize and generalize the following well-known example of differentiating a Gaussian process.

\begin{example}\label{example_derivative_GP}
	Let $g=\GP(0,k(x,x'))$ be a scalar univariate Gaussian process with differentiable realizations.
	Then,
	\begin{align*} %\label{eq_differentiable}
	\begin{bmatrix}\frac{\partial}{\partial x}\end{bmatrix}_*g := \GP\left(0,\frac{\partial^2}{\partial x\partial x'}k(x,x')\right)\mbox{}
	\end{align*}
	is the Gaussian process of derivatives of realizations of the Gaussian process $g$.
	One can interpret this Gaussian process $\begin{bmatrix}\frac{\partial}{\partial x}\end{bmatrix}_*g$ as taking derivatives as measurement data and producing a regression model of derivatives.
	Taking a one-sided derivative $\frac{\partial}{\partial x}k(x,x')$ yields the cross-covariance between a function and its derivative.
	See \citep[\textsection5.2]{StationaryAndRelatedStochasticProcesses} for a proof and \citep{wu2017bayesian} resp.\ \cite{cobb2018identifying} for a applications in Bayesian optimization resp.\ vector field modeling.
\end{example}

Given a set of functions $G\subseteq \{f:X\to Y\}$ and $b:Y\to Z$, then the \emph{pushforward} is
\begin{align*}
b_*G=\{b\circ f\mid f\in G\}\subseteq \{f:X\to Z\}.
\end{align*}
%\begin{wrapfigure}{l}{3.3cm}
%	\vspace{-1em}
\begin{center}
	\begin{tikzcd}
		X\arrow[rr, bend left, "b_*G"]\arrow[r, "G"] & Y\arrow[r,"b"] & Z
	\end{tikzcd}
\end{center}
%	\!\!\!\!\!\!
%\end{wrapfigure}
A \emph{pushforward} of a stochastic Process $g:\Omega\to(X\to Y)$ by $b:Y\to Z$ is
\begin{align*}
b_*g:\Omega\to(X\to Z):\omega\mapsto(b\circ g(\omega)).
\end{align*}

\begin{lemma}\label{lemma_pushforward_gaussian}
	Let $\F$ and $\G$ be spaces of functions defined on $X\subseteq\R^d$ with product  $\sigma$-algebra of function evalutions.
	Let $g=\GP(\mu(x),k(x,x'))$ with realizations in $\F$ and $B:\F\to\G$ a linear, measurable operator which commutes with expectation w.r.t.\ the measure induced by $g$ on $\F$ and by $B_*g$ on $\G$.
	Then, the pushforward $B_*g$ of $g$ under $B$ is again Gaussian with
	\begin{align*}%\label{eq_tansformation}
	B_*g = \GP(B\mu(x),Bk(x,x')(B')^T)\mbox{ ,}
	\end{align*}
	where $B'$ denotes the operation of $B$ on functions with argument $x'$.
\end{lemma}

Call $B_*g$ the \emph{pushforward Gaussian process} of $g$ under $B$.

We postpone the proof to the appendix. %Appendix~\ref{section_proof_lemma_pushforward_gaussian}.
Lemma~\ref{lemma_pushforward_gaussian} is often stated without assuming that $B$ commutes with expectation, but also without proof.
If such a more general version of Lemma~\ref{lemma_pushforward_gaussian} holds, the author would be very interested to see a reference.
Special cases have been discussed in the literature, often only for mean square differentiability (\citealp{papoulis2002}, after (9.87) resp.\ 
(10.78);
in the first and second resp.\ third edition;
 \citealp[Thm~2.2.2]{adler1981geometry}, \citealp{agrell2019gaussian}; \citealp[\textsection2.3]{da2012gaussian}; \citealp[Thm.~9]{RKHSProbabilityStatistics}).

Consider change points and change surfaces as application of Lemma~\ref{lemma_pushforward_gaussian}, following \citep{garnett2009sequential,garnett2010sequential,lloyd2014automatic,herlands2016scalable}.

\begin{example}\label{example_partition_unity}
	Let $\rho_1,\rho_2:\R^d\to[0,1]$ a \emph{partition of unity},
	%\footnote{Generalizing to more than two summands is straight forward.}
	i.e., $\rho_1(x)+\rho_2(x)=1$ for all $x\in\R^d$.
	Usually, both $\rho_1$ and $\rho_2$ are close to being $0$ or close to being $1$ over most of $\R^d$.
	Such a partition of unity induces a linear operator
	\begin{align*}
	\rho=\begin{bmatrix}\rho_1 & \rho_2\end{bmatrix}:\F^{2\times 1}\to\F^{1\times1}:\begin{bmatrix}f_1\\f_2\end{bmatrix}\mapsto \begin{bmatrix}\rho_1 & \rho_2\end{bmatrix}\begin{bmatrix}f_1\\f_2\end{bmatrix},
	\end{align*}
	where $\F$ is a space of functions $\R^d\to\R$.
	Given two independent Gaussian processes $g_1=\GP(0,k_1)$, $g_2=\GP(0,k_2)$ with realizations in $\F$, we have
	\begin{align*}
	&\rho_*g:=\begin{bmatrix}\rho_1 & \rho_2\end{bmatrix}_*\begin{bmatrix}g_1\\g_2\end{bmatrix}\\
	&=\GP\left(0,\rho_1(x)k_1(x,x')\rho_1(x')+\rho_2(x)k_2(x,x')\rho_2(x')\right)
	\end{align*}
	Thereby, we model change points (for $d=1$) or change surfaces (for $d>1$) at positions where $\rho_1$ changes from being close to 0 to being close to 1.
	This example is the basis for boundary conditions in Section~\ref{section_boundary}:
	when setting $g_2$ to zero, $\rho_*g$ is close to zero where $\rho_2\approx1$ and close to $g_1$ where $\rho_1\approx1$.
\end{example}

\section{Solution Sets of Operator Equations}\label{section_GPs_solutions}

We consider linear ordinary and partial differential equations defined on the set of smooth functions.
Let $\F=C^\infty(X,\R)$ be the real vector space of smooth functions from $X\subseteq\R^d$ to $\R$ with the usual Fr\'echet topology\footnote{For Gaussian processes on Fr\'echet spaces see \citep{ZapalaGPFrechet,osswald2012malliavin}. The topology is generated by the separating family
%	\begin{align*}
	$
	\|f\|_{a,b}:=
%	\sup_{\substack{i\in\Z_{\ge0}^d\\ |i|\le a}}
	\sup_{i\in\Z_{\ge0}^d,|i|\le a}
	\sup_{z\in [-b,b]^d}\ |\frac{\partial}{\partial z^i}f(z)|
	$
%	\end{align*}
	of seminorms for $a,b\in\Z_{\ge0}$ on $\F$ \citep[\S10]{TopologicalVectorSpaces}.}.
The \emph{squared exponential covariance function}
\begin{align}\label{eq_SE}
k_\F(x_i,x_j)=\exp\left(-\frac{1}{2}\sum_{a=1}^d(x_{i,a}-x_{j,a})^2\right)
\end{align}
induces a Gaussian process prior $g_\F=\GP(0,k_\F)$ with realizations dense in the space of smooth functions $\F=C^\infty(X,\R)$ w.r.t.\ this topology.

The following three \emph{rings of linear operators} $R$ model operator equations.
These rings are $\R$-algebras s.t.\ $\F$ is a \emph{topological $R$-(left-)module}, i.e., 
$\F$ is a topological $\R$-vector space of functions $X\to\R$ for $X\subseteq\R^d$ that also is an $R$-(left-)module such that the elements of $R$ operate continuously on $\F$.

\begin{example}\label{example_linear_dgl}
	The polynomial\footnote{Partial derivatives commute (symmetry of 2nd derivatives) and generate a \emph{commutative} polynomial ring.}
	ring $R=\R[\partial_{x_1},\ldots,\partial_{x_d}]$ models linear differential equations with constant coefficients, as $\partial_{x_i}$ acts on $\F=C^\infty(X,\R)$ via partial derivative w.r.t.\ $x_i$.%, making $\F$ into an $R$-module.
\end{example}

\begin{example}\label{example_polynomial_operator_ring}
	The polynomial ring $R=\R[x_1,\ldots,x_d]$ models algebraic equations via multiplication on $\F$.
	This ring is relevant for boundary conditions.
\end{example}

To combine linear differential equations with constant coefficients with boundary conditions or to model linear differential equations with polynomial\footnote{No major changes for rational, holonomic, or meromorphic coefficients.} coefficients, consider the following ring.

\begin{example}\label{example_Weyl}
	The Weyl algebra $R=\R[x_1,\ldots,x_n]\langle \partial_{x_1},\ldots,\partial_{x_n}\rangle$ has the non-commutative relation $\partial_{x_i}x_j=x_j\partial_{x_i}+\delta_{ij}$ representing the product rule of differentiation, where $\delta_{ij}$ is the Kronecker delta.
\end{example}

Operators defined over these three rings satisfy the assumptions of Lemma~\ref{lemma_pushforward_gaussian}:
multiplication commutes with expectations and the dominated convergence theorem implies that expectation commutes with derivatives, as realizations of $g_\F$ are continuously differentiable.
These three rings also operate continuously on $\F$:
the Fr\'echet topology is constructed to make derivation continuous, and multiplication is bounded (if $X$ is bounded and bounded away from infinity) and hence continuous, as $\F$ is Fr\'echet.

\subsection{Parametrizations}\label{subsection_paramerization}

% This Bayesian standpoint gives a good interpretation of a second way to operate on Gaussian processes by applying the operator to the covariance function from only one side.
% This corresponds to applying the operator only to the regression model.

% In contrast, applying $B$ to a Gaussian process and interpreting the result again as function yields\footnote{%
%   In the following, we do not need that $\H(G)$ is closed under $R$, even though this holds for derivatives \cite[Thm~1.1]{zhou2008derivative}.
%   Under the mild assumption of $R$ having an ascending filtration $R_{\le n}$ it suffices to work in the directed set $R_{\le i}\H(G)\hookrightarrow R_{\le i+1}\H(G)$ of RKHSs.
% } $\H^0(B\GP(0,k))=B\H^0(\GP(0,k))$ and $\H(B\GP(0,k))=B\H(\GP(0,k))$.

For $A\in R^{\ell'\times\ell}$ define the \emph{solution set}
$
\sol_\F(A):=\{f\in \F^{\ell\times1}\mid Af=0\}
$
as a nullspace of an operator matrix $A$.
We say that a Gaussian process is \emph{in} a function space, if its realizations are contained in said space.
The following tautological lemma is a version of the fundamental theorem of homomorphisms.
It describes the interplay of Gaussian processes and solution sets of operators.

\begin{lemma}[{\citealp[Lemma~2.2]{LH_AlgorithmicLinearlyConstrainedGaussianProcesses}}]\label{lemma_gp_operator}
	Let $g=\GP(\mu,k)$ be a Gaussian process in $\F^{\ell\times1}$.
	Then $g$ is a Gaussian process in the solution set $\sol_\F(A)$ of $A\in R^{\ell'\times\ell}$ if and only if both $\mu$ is contained in $\sol_\F(A)$ and $A_*(g-\mu)$ is the constant zero process.
\end{lemma}
To construct Gaussian processes with realizations in the solution set $\sol_\F(A)$ of an operator matrix $A\in R^{\ell'\times\ell}$, one looks for a $B\in R^{\ell\times\ell''}$ with $AB=0$ \citep{LinearlyConstrainedGP}.
If $g=\GP(0,k)$ is a Gaussian process in $\F^{\ell''\times 1}$, then the realizations of $B_*g$ are contained in $\sol_\F(A)$ by Lemma~\ref{lemma_gp_operator}, as $A_*(B_*g)=(AB)_*g=0_*g=0$.
In practice, one would like that any solution in $\sol_\F(A)$ can be approximated by $B_*g$ to arbitrary precision, i.e., that the realizations of the Gaussian process $B_*g$ are dense in $\sol_\F(A)$.
To this end, we call $B\in R^{\ell\times\ell''}$ a \emph{parametrization} of $\sol_\F(A)$ if $\sol_\F(A)=B\F^{\ell''\times 1}$.

\begin{proposition}[{\citealp[Prop.~2.4]{LH_AlgorithmicLinearlyConstrainedGaussianProcesses}}]\label{proposition_denseparametrization}
	Let $B\in R^{\ell\times\ell''}$ be a parametrization of $\sol_\F(A)$ for $A\in R^{\ell'\times\ell}$.
	Take the Gaussian process $g_\F^{\ell''\times 1}$ of $\ell''$ i.i.d.\ copies of $g_\F$, the Gaussian process with squared exponential covariance\footnote{Or any other covariance with realizations dense in $\F$.} function $k_\F$ from Eq.~\eqref{eq_SE}.
	Then, the realizations of $B_*g_\F^{\ell''\times 1}$ are dense in $\sol_\F(A)$.
\end{proposition}
\begin{proof}
	By construction, realizations of $g_\F^{\ell''\times 1}$ are dense in $\F^{\ell''\times 1}$.
	The operator $B$ induces a surjective continuous ($\F$ is a topological $R$-module) map.
	Surjective continuous maps map dense sets to dense sets.
\end{proof}

\subsection{Algorithmically constructing parametrizations}\label{subsection_Parametrizations}

We summarize the algorithm which decides whether a parametrization of a system of linear differential equations exists and compute it in the positive case.
To construct the parametrization $B$, we are lead to just compute the nullspace\footnote{We avoid calling nullspaces kernel, due to confusion with symmetric positive semidefinite functions. While a left resp.\ right nullspace is a module, we abuse notation and denote any matrix as left resp.\ right nullspace if its rows resp.\ columns generate the nullspace as an $R$-module.} of
$
\F^{\ell'\times 1}\xleftarrow{A}\F^{\ell\times 1}\mbox{.}
$
This is not feasible, as $\F$ is too ``big'' to allow computations.
Instead, we compute the nullspace of
$
R^{\ell'\times 1}\xleftarrow{A}R^{\ell\times 1}\mbox{,}
$
a symbolic computation, only using operations over $R$ without involvement of $\F$.
% If $\F$ is injective, the nullspace of equation \eqref{equation_nullspace_R} gives a nullspace of equation \eqref{equation_nullspace_F}.

\begin{theorem}\label{theorem_parametrizable}
	Let $A\in R^{\ell'\times\ell}$.
	Let $B$ be the right nullspace of $A$ and $A'$ the left nullspace of $B$.
	Then $\sol_\F(A')$ is the largest subset of $\sol_\F(A)$ that is parametrizable and $B$ parametrizes $\sol_\F(A')$.
\end{theorem}

A well-known and trivial special case of this theorem are linear equations in finite dimensional vector spaces, with $R=\F=\R$ the field of real numbers.
In that case, $\sol_\F(A)$ can be found by applying the Gaussian algorithm to the homogeneous system of linear equations $Ab=0$ and write a base for the solutions of $b$ as columns of a matrix $B$.
This matrix $B$ is the (right) nullspace of $A$.
There are no additional equations
%\footnote{F.d.\ vector spaces are naturally isomorphic to their bi-dual.}
 satisfied by the above solutions, i.e.\ $A=A'$ generates the (left) nullspace of $B$.

In general, the left nullspace $A'$ of the right nullspace $B$ of $A$ is not necessarily $A$.
E.g., for the univariate polynomial ring $R=\R[x]$ and the matrix $A=\begin{bmatrix}x\end{bmatrix}$ we have $B=\begin{bmatrix}0\end{bmatrix}$ and $A'=\begin{bmatrix}1\end{bmatrix}$.

\begin{corollary}\label{corollary_parametrizable}
	In Theorem~\ref{theorem_parametrizable}, $\sol_\F(A)$ is parametrizable if and only if the rows of $A$ and $A'$ generate the same row-module.
	Since $AB=0$, this is the case if all rows of $A'$ are contained in the row module generated by the rows of $A$.
	In this case, $\sol_\F(A)$ is parametrized by $B$. 
\end{corollary}

For a formal proof we refer to the literature (\citealp[Thm.~2]{ZSHinverseSyzygies}; \citealp[Thm.~3, Alg.~1, Lemma~1.2.3]{Zerz}; \citealp[\S7.(24)]{Ob}; \citealp{QGrade,Q_habil,BaPurity_MTNS10,SZinverseSyzygies,CQR05_nonote,RobRecentProgress}).
Luckily, there is a high level description of the parametrizable systems.
\begin{theorem}[{\citealp[\S7.(21)]{Ob}}]\label{theorem_controllable_parametrizable}
	A system $\sol_\F(A)$ is parametrizable iff it is controllable.
\end{theorem}	
The intuition for controllability is that one can partition the functions of the system into state and input, such that any chosen state can be reached by suitably manipulating the inputs.
In particular, controllable systems (except the trivial system) are far away from having a unique solution.
If $A$ is not parametrizable, then the solution set $\sol_\F(A')$ is the subset of controllable behaviors in $\sol_\F(A)$.

%In the recent decades, Gr\"obner bases algorithms have become one of the core algorithms of computer algebra, with manifold applications in geometry, system theory, natural sciences, automatic theorem proving, post-quantum cryptography, and many others.
Reduced Gr\"obner bases generalize the reduced echelon form from linear systems to systems of polynomial (and hence linear operator) equations, by bringing them into a standard form.
%\footnote{depending only on the choice of a so-called monomial order}.
	%, similar to the ordering of the variables in the Gaussian algorithm
They are computed by Buchberger's algorithm, which is a generalization of the Gaussian and Euclidean algorithm and a special case of the Knuth-Bendix completion algorithm for rewriting systems.
The generalization of Gr\"obner bases to vectors of polynomials is straight forward.

Gr\"obner bases make the above theorems algorithmic.
Similar to the reduced echelon form, Gr\"obner bases allow to compute all solutions over $R$ of the homogeneous system and compute, if it exists, a particular solution over $R$ for an inhomogeneous system.
Solving homogeneous systems is the same as computing its right resp.\ left nullspace (of $A$ resp.\ $B$).
Solving inhomogeneous equations decides whether an element (the rows of $A'$) is contained in a module (the row module of $A$).
A formal description of Gr\"obner bases exceeds the scope of this note.
We refer to the excellent literature \citep{SturmfelsWhatIs,eis,al,GP08,GerI,Buch}.
 Not only do they generalize the Gaussian algorithm for linear polynomials, but also the Euclidean algorithm for univariate polynomials.
In addition to polynomial rings, Gr\"obner bases also exist for the Weyl algebra \citep{robphd,JO,CQR07,LevThesis,plural} and many further rings.
The algorithms used in the paper are usually readily available functions implemented in various computer algebra systems \citep{singular412,M2}.
While Gröbner bases depend on the choice of a term order, similar to reordering columns in the Gaussian algorithm, any term order leads to correct results.

Gr\"obner bases solve problems of high complexity like \textsc{ExpSpace} completeness \citep{Mayr,mayrMeyer,BS88}.
In practice, this is less of a problem, as the Gr\"obner basis computations only involve the operator equations, but no data.
Hence we view the complexity of the Gr\"obner basis computations in $\mathcal{O}(1)$, which only needs to be applied once to construct the covariance function.
In particular, the Gröbner bases of every example in this paper terminate instantaneously.
For larger examples, the data dependent $\mathcal{O}(n^3)$ of the Gaussian processes is the computationally restricting subalgorithm.

%The following example is extended over the course of this paper.

\begin{example}[{\citealp[Example 4.4]{LH_AlgorithmicLinearlyConstrainedGaussianProcesses}}]\label{example_sphere2}
	We construct a prior for smooth tangent fields on the sphere without sources and sinks using the polynomial Weyl algebra $R=\R[x,y,z]\langle \partial_x, \partial_y, \partial_z\rangle$.
	I.e., we are interested in $\sol_A(\F)=\{v\in C^\infty(S^2,\R^3)\mid Av=0\}$ for
	\begin{align*}
	A:=\begin{bmatrix}
	x & y & z \\
	\partial_x & \partial_y & \partial_z \\    	
	\end{bmatrix}\mbox{.}
	\end{align*}
	The right nullspace
	\begin{align*}
	B:=\begin{bmatrix}
	-z\partial_y+y\partial_z\\
	z\partial_x-x\partial_z\\
	-y\partial_x+x\partial_y\\
	\end{bmatrix}\mbox{.}
	\end{align*}
	can be checked to yield a parametrization of $\sol_\F(A)$.
	For a demonstration of this covariance functions, see Figure~\ref{figure_example_sphere}.
		
	\end{example}

\section{Boundary conditions}\label{section_boundary}

Differential equations and boundary conditions go hand in hand in applications.
Here, we recall a general methods to incorporate boundary conditions into Gaussian processes, a slight generalization of \citep[Section~3]{NoisyLinearOperatorEquationsGP}, closely related to vertical rescaling.
Boundary conditions in ODEs are equivalent to conditioning on data points \cite{GOODE}.
%Recall from Example~\ref{example_derivative_GP} that one can easily construct a covariance function between a Gaussian process, its derivatives, and higher order Taylor coefficients, under mild assumptions.
% This approach is well-known. %, we limit ourselves to a simple example.

%\begin{example}\label{example_ode_boundary}
%	Consider for a function $f$ following a zero mean Gaussian process prior with squared exponential covariance function from Equation~\eqref{eq_SE} and 
%	\begin{align*}
%	f(0)=f'(0)=f(1)=f'(1)=0.
%	\end{align*}
%	
%	The resulting posterior has mean zero and covariance
%%	\begin{align*}
%%		&{\rm e}^{-\frac12(x-y)^2}
%%		-{\frac{{\rm e}^{-\frac12x^2-\frac12y^2}}{{{\rm e}^{-2}}-3\,{{\rm e}^{-1}}+1}}\cdot\\*
%%		&\Big(
%%		(xy+1)
%%		+(xy-x-y+2)	{{\rm e}^{x+y-1}}\\
%%		& 	+(-2xy+x+y-1) \left({{\rm e}^{x+y-2}}+{{\rm e}^{-1}}\right)\\
%%		& 	+(xy-y+1)		{{\rm e}^{y-2}}
%%		+(xy-x+1)		{{\rm e}^{x-2}}\\
%%		& 	+(y-x-2) 		{{\rm e}^{y-1}}
%%		+(x-y-2) 		{{\rm e}^{x-1}}	
%%		\Big),
%%	\end{align*}
%	\begin{align*}
%	&\mathrm{e}^{-\frac12(x-y)^2}
%	-{\frac{\mathrm{e}^{-\frac12x^2-\frac12y^2}}{{\mathrm{e}^{-2}}-3\,{\mathrm{e}^{-1}}+1}}\cdot
%	\Big(
%	(xy-x-y+2)	{\mathrm{e}^{x+y-1}}
%	+(xy-x+1)		{\mathrm{e}^{x-2}}
%	+(y-x-2) 		{\mathrm{e}^{y-1}}\\
%	&+(x-y-2) 		{\mathrm{e}^{x-1}}	
%	+(xy-y+1)		{\mathrm{e}^{y-2}}
%	+(xy+1)
%	+(-2xy+x+y-1) \left({\mathrm{e}^{x+y-2}}+{\mathrm{e}^{-1}}\right)
%	\Big).
%	\end{align*}
%	% which is zero if $x$ or $y$ are in $\{0,1\}$.
%\end{example}

We recall the creation of priors for homogeneous boundary conditions for PDEs from \cite{NoisyLinearOperatorEquationsGP}, for the inhomogeneous case see Section~\ref{section_inhomogeneous}.
Such boundary conditions fix the function values and/or their derivatives at a subset of the domain $X$ exactly.
We restrict ourselves to zero sets of polynomials.
For more complicated, approximate boundary conditions see \cite{solin2019know} and for asymptotic boundaries see \cite{tan2018gaussian}.

Denote again by $\F=C^\infty(X,\R)$ the set of smooth functions defined on $X\subset\R^d$ compact.
Let $R'\subset \R^{X}$ be a Noetherian ring of functions and subring of $R$, and $M\subseteq X$ implicitely defined 
\begin{align*}
M=\V(I):=\left\{m\in X\,\middle|\,f(m)=0\mbox{ for all }f\in I\right\}
\end{align*}
for an ideal $I\trianglelefteq R'$ of equations.
An important example for this setting is the Weyl algebra $R=\R[x_1,\ldots,x_d]\langle \partial_{x_1},\ldots,\partial_{x_d}\rangle$ and its subring $R'$ the polynomial ring $R'=\R[x_1,\ldots,x_d]$.

%\begin{proposition}
%All solutions of a homogenous boundary condition $f_{|M}=0$ for a single function $f\in\F$ can be parametrized by
%%\begin{align*}
%$\F\xleftarrow{B'}\F^{\ell''\times1}$
%%\end{align*}
%where $B'=\begin{bmatrix}f_1 & \ldots & f_{\ell''}\end{bmatrix}$ is a row whose entries generate $I$.
%\end{proposition}
\begin{proposition}\label{prop_boundary}
	A row  $B'=\begin{bmatrix}f_1 & \ldots & f_{\ell''}\end{bmatrix}$ whose entries generate the ideal $I$ parametrizes all solutions of a homogenous boundary condition $f_{|M}=0$ for a function $f\in\F$ via 
%\begin{align*}
$\F\xleftarrow{B'}\F^{\ell''\times1}$
%\end{align*}
\end{proposition}
\begin{proof}
	Let on the one hand $p\in\R^d$ such that $f_i(p)=0$ for all $1\le i\le \ell''$.
	Then, $(B'g)(p)=0$ for all $g\in \F^{\ell''\times1}$.
	On the other hand, let $p\in\R^d$ such that there is an $1\le j\le \ell''$ with  $f_j(p)\not=0$ and parametrize $h\in\F$ locally as $h(x)=B'\cdot \frac{h(x)}{f_j(x)}e_j$
	%\begin{align*}
		%h(x)=B'\cdot \frac{h(x)}{f_j(x)}e_j.
		%\begin{bmatrix}0 \\ \ldots \\ 0 \\ \frac{g(x)}{f_j(x)} \\ 0 \\ \ldots \\ 0\end{bmatrix}
	%\end{align*}
	for $e_j$ the $j$th standard basis vector, since locally $f_j(x)\not=0$.
	For a global parametrization, patch the local parametrizations via a partition of unity.
\end{proof}

To encode boundary conditions for $\ell>1$ functions, we use a direct sum matrix $B'\in (R')^{\ell\times\ell\ell''}$, e.g.,  $B'=\begin{bmatrix}B'_1&0\\0&B'_2\end{bmatrix}$ for $\ell=2$ where $B_1'$ and $B_2'$ are rows over $R'$ describing the boundaries. % for the first resp.\ second component.

\begin{example}\label{example_dirichlet}
	Functions $\F=C^\infty([0,1]^2,\R)$ with Dirichlet boundary conditions 
	$
	f(0,y)=f(1,y)=f(x,0)=f(x,1)=0
	$
	are parametrized by 
	$B'=\begin{bmatrix} x(x-1)y(y-1) \end{bmatrix}$.
\end{example}

\begin{example}
	Functions $\F=C^\infty(\R^3,\R)$ with boundary condition $f(0,0,z)=0$ are parametrized by 
	$B'=\begin{bmatrix} x & y \end{bmatrix}$.
\end{example}

%Generalized boundary conditions with derivatives vanishing can be achieved by multiplicities in the ideal $I$.

\begin{example}\label{example_boundary_derivative}
	Consider $\F=C^\infty(\R^2,\R)$ with boundary conditions
	%\begin{align*}
	$f(0,y)=\left(\frac{\partial}{\partial x}f(x,y)\right)_{|x=0}=0\mbox{.}$
	%\end{align*}
	Such functions are parametrized by $B=\begin{bmatrix} x^2 \end{bmatrix}$, since
	\begin{align*}
		&\left(\frac{\partial}{\partial x}(x^2f(x,y))\right)_{|x=0}\\
		=& 
		\left(2xf(x,y)+x^2\frac{\partial}{\partial x}f(x,y)\right)_{|x=0}
		=
		0.
	\end{align*}
\end{example}

\section{Intersecting parametrizations}\label{section_boundary_combination}

Now, we combine parametrizations $B_1\in R^{\ell\times\ell''}$ and $B_2\in R^{\ell\times\ell'''}$, e.g.\ from differential equations and boundary conditions, by intersecting their images $B_1\F^{\ell''}\cap B_2\F^{\ell'''}$.

\begin{example}
	Actually, the Dirichlet boundary condition of Example~\ref{example_dirichlet} is an intersection of the images of the boundary conditions parametrized by $\begin{bmatrix} x \end{bmatrix}$, $\begin{bmatrix} x-1 \end{bmatrix}$, $\begin{bmatrix} y \end{bmatrix}$, and $\begin{bmatrix} y-1 \end{bmatrix}$.
\end{example}

The following theorem is the main contribution of this paper.
It constructs a parametrization of intersections of parametrizations algorithmically.
\begin{theorem}[Intersecting parametrizations]\label{theorem_combining_parametrizations}
	Let $B_1\in R^{\ell\times\ell_1''}$ and $B_2\in R^{\ell\times\ell_2''}$.
	Denote by 
	\begin{align*}
	C:=\begin{bmatrix} C_1 \\ C_2\end{bmatrix}\in R^{(\ell_1''+\ell_2'')\times m}
	\end{align*}
	the right-nullspace of the matrix $B:=\begin{bmatrix} B_1 & B_2\end{bmatrix}\in R^{\ell\times(\ell_1''+\ell_2'')}$.
	Then $B_1C_1=-B_2C_2$ parametrizes solutions of $B_1\F^{\ell_1''}\cap B_2\F^{\ell_2''}$.
\end{theorem}
For the proof cf.\ the appendix. %Appendix~\ref{section_proof_theorem_combining_parametrizations}.
The computations are again Gr\"obner basis computations over the ring $R$.
%The computations are mostly algorithmic.

\begin{example}\label{example_sphere2b}
	We rephrase the computation of divergence free fields on the sphere from Example~\ref{example_sphere2}.
	This is the intersection of divergence free fields, the zero set of $A_1:=\begin{bmatrix}
	\partial_x & \partial_y & \partial_z \\    	
	\end{bmatrix}$, 
	and the fields on the sphere, the zero set of
	$A_2:=\begin{bmatrix}
	x & y & z \\    	
	\end{bmatrix}$, 
	respectively parametrized by	
	\begin{align*}
	B_1=\begin{bmatrix}
	0 & \partial_z & -\partial_y \\
	-\partial_z & 0 & \partial_x \\
	\partial_y & -\partial_x & 0 \\
	\end{bmatrix}
	\text{and }
	B_2=\begin{bmatrix}
	0 & z & -y \\
	-z & 0 & x \\
	y & -x & 0 \\
	\end{bmatrix}.
	\end{align*}
	
	The right-nullspace of $\begin{bmatrix}B_1 & B_2\end{bmatrix}$ is
	\begin{align*}
	C=
	\begin{bmatrix}
	C_1\\C_2
	\end{bmatrix}
	=
	\begin{bmatrix}
	x&\partial_x&0\\
	y&\partial_y&0\\
	z&\partial_z&0\\
	\partial_x&0&x\\
	\partial_y&0&y\\
	\partial_z&0&z	
	\end{bmatrix}
	\end{align*}
	The matrix $\begin{bmatrix}B_1 & B_2\end{bmatrix}$ is the left nullspace of $C$.
	Now, 
	\begin{align*}B_1C_1=-B_2C_2=
	\begin{bmatrix}
	z\partial_y-y\partial_z&0&0\\ 
	-z\partial_x+x\partial_z&0&0\\
	y\partial_x-x\partial_y&0&0
	\end{bmatrix}
	\end{align*}
	is equivalent\footnote{The matrices $B_1$ and $B_2$ each have a non-zero nullspace, corresponding to the two trivial columns in $B_1C_1$.} to the matrix $B$ from Example~\ref{example_sphere2}.	
\end{example}

%Whereas Corollary~\ref{corollary_parametrizable} gives a necessary and sufficient condition whether a parametrization exists, this theorem only gives a sufficient condition.
%This is due to parametrizations being more general than parametrizations of controllable systems. %, e.g., no example in Section~\ref{section_boundary} is the parametrization of a controllable system. %, as the left-nullspace of any parametrizing matrix $B'$ is zero.
%Hence, we might get a parametrization, even if the left nullspace of $C$ is bigger than $B$.

\begin{example}\label{example_sphere_boundary}
	We continue with the divergence free fields on the sphere from Examples~\ref{example_sphere2} and \ref{example_sphere2b}.
	These are parametrized by
	\begin{align*}
	B_1:=\begin{bmatrix}
	-z\partial_y+y\partial_z\\
	z\partial_x-x\partial_z\\
	-y\partial_x+x\partial_y\\
	\end{bmatrix}\mbox{ .}
	\end{align*}
	functions vanishing at the equator (boundary condition: $f(x,y,0)=0$) are parametrized by 
	\begin{align*}
	B_2:=\begin{bmatrix}
	z & 0 & 0 \\
	0 & z & 0 \\
	0 & 0 & z \\
	\end{bmatrix}\mbox{ .}
	\end{align*}
	The nullspace of $\begin{bmatrix}B_1 & B_2\end{bmatrix}$ is
	\begin{align*}
	C:=
	\begin{bmatrix}
	C_1\\C_{2,1}\\C_{2,2}\\C_{2,3}
	\end{bmatrix}
	=		
	\begin{bmatrix}
	-{z}^{2}\\
	z^2\partial_y-yz\partial_z-2y\\
	-z^2\partial_x+xz\partial_z+2\,x\\
	yz\partial_x-xz\partial_y
	\end{bmatrix}.
	\end{align*}
	The left nullspace of $C$ is not only generated by $\begin{bmatrix}B_1 & B_2\end{bmatrix}$, but by the additional relation
	$
	D:=
	\begin{bmatrix}
	0 & x & y & z
	\end{bmatrix}\mbox{.}
	$
	This relation $D$ tells us, that the parametrized solutions of $C_2$ are a vector field on a sphere around the origin, which they remain after being multiplied by the scalar matrix $B_2$.
	We gladly accept this additional condition.
	Now, 
	\begin{align*}
	B_1C_1=
	-B_2C_2=
	\begin{bmatrix}
	-z^3\partial_y+yz^2\partial_z+2yz\\
	z^3\partial_x-xz^2\partial_z-2xz \\
	-yz^2\partial_x+xz^2\partial_y\end{bmatrix}
	\end{align*}
	parametrizes the divergence free fields on the sphere vanishing at the equator, see Figure~\ref{figure_example_sphere3}.
\end{example}

\begin{remark}\label{remark_graepl}
	\citep{NoisyLinearOperatorEquationsGP} also constructs a Gaussian process prior for a system $Af=y$ of linear differential equations with boundary conditions.
	It assumes any Gaussian process prior on $f$ and uses a variant of Lemma~\ref{lemma_pushforward_gaussian} to compute the cross-covariance between $y$ and $f$, which allows to condition the model $p(f)$ on data for $y$.
	This model ensures in no way that $f$ is constrained to solutions of $Af=y$, even if e.g.\ $y=0$ is known.
	Furthermore, conditioning $p(f)$ on data for $f$ is just done w.r.t.\ the (uninformative) Gaussian process prior chosen for $f$.
	
	%The approach to inclusion of boundary conditions in \citep{NoisyLinearOperatorEquationsGP} and this paper also shares some similarity.
	As in this paper, \citep{NoisyLinearOperatorEquationsGP} uses Proposition~\ref{prop_boundary} and a pushforward to construct a prior for $f$ supported on solutions of the homogeneous boundary condition.
	No effort to combine differential equations and boundary conditions as in Theorem~\ref{theorem_combining_parametrizations} is necessary, since the differential equations are not satisfied anyway.
	The case of inhomogeneous boundary conditions is solved via taking a particular solutions as a mean function.
	Finding such a particular solution is simple, as only the boundary conditions must be satisfied; in contrast to Section~6 of this paper, where also the differential equations need to be satisfied.
	
	% The advantages in this paper are only possible by restricting to controllable systems.

%	The fundamental idea to solve the above question is to consider the cross covariance between $f$ and $y$, which is $Ak_f$ for the covariance $k_f$ of $f$.
%	This approach approach allows to incorporate linear and (for the sake of simplicity, the inhomogeneous case is possible with similar restrictions to finding a particular solution as Section~\ref{section_inhomogeneous}) homogeneous boundary conditions $Df=0$.
%	Therefore, one considers a matrix $B_2$ parametrizing\footnote{\citet{NoisyLinearOperatorEquationsGP}, similar to \citet{LinearlyConstrainedGP}, forget to demand surjectivity.} the boundary condition, i.e., $DB_2=0$.
%	Now, \citet{NoisyLinearOperatorEquationsGP} uses a parametrizing function $g$ and sets $f=B_2g$.
%	This ensures $Df=0$ and allows to construct a covariance between $f$ and $y$ by some operator algebra.
	%	The setup of \citep{NoisyLinearOperatorEquationsGP} is actually a special case of \citet{LH_AlgorithmicLinearlyConstrainedGaussianProcesses}.
	%	Therefore, note that $Af=y$ is equivalent to the homogeneous and controllable system $\widetilde{A}\begin{bmatrix} f\\y\end{bmatrix}=0$ for $\widetilde{A}=\begin{bmatrix} A & -I\end{bmatrix}$ for the identity matrix $I$.
	%	The parametrization of $\widetilde{A}$ is $B_1:=\begin{bmatrix} I\\ A\end{bmatrix}$.
	%	Due to $I$ appearing in $B_1$, the intersection of the parametrizations $B_1$ and $B_2$
\end{remark}

\section{Inhomogenous boundary conditions}\label{section_inhomogeneous}

So far, we have only considered homogeneous equations and boundary conditions, i.e., with right hand sides zero.
The fundamental theorem of homomorphisms (cf.\ Lemma~\ref{lemma_gp_operator}) extends this to the inhomogeneous case, by taking a particular solution as mean function.
While simple theoretically, finding a particular solution can be quite hard in practice.
We restrict ourselves to examples.

\begin{figure*}
	\centering
	\centerline{
		
		\begin{minipage}{0.33\textwidth}
			\vspace{1.0em}
			\includegraphics[width=\textwidth]{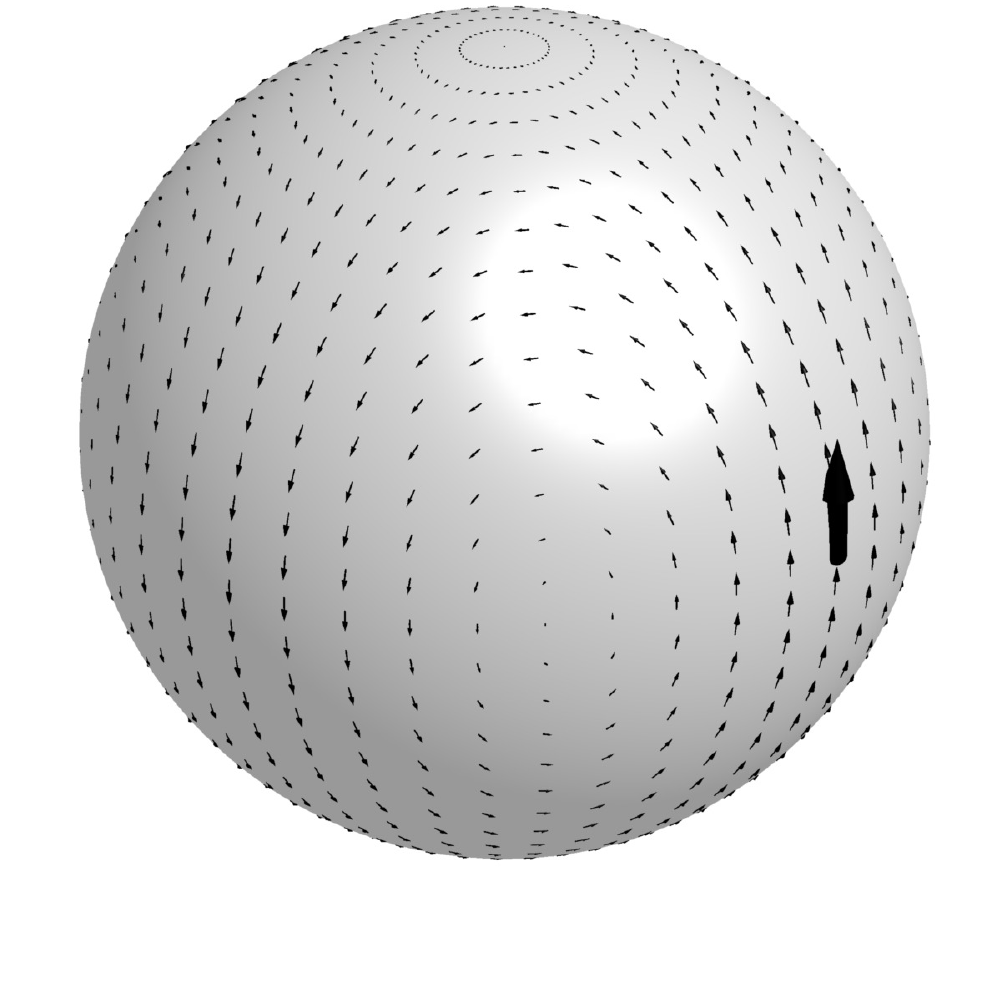}
		\end{minipage}
		\begin{minipage}{0.33\textwidth}
			\vspace{-2.0em}
			\includegraphics[width=\textwidth]{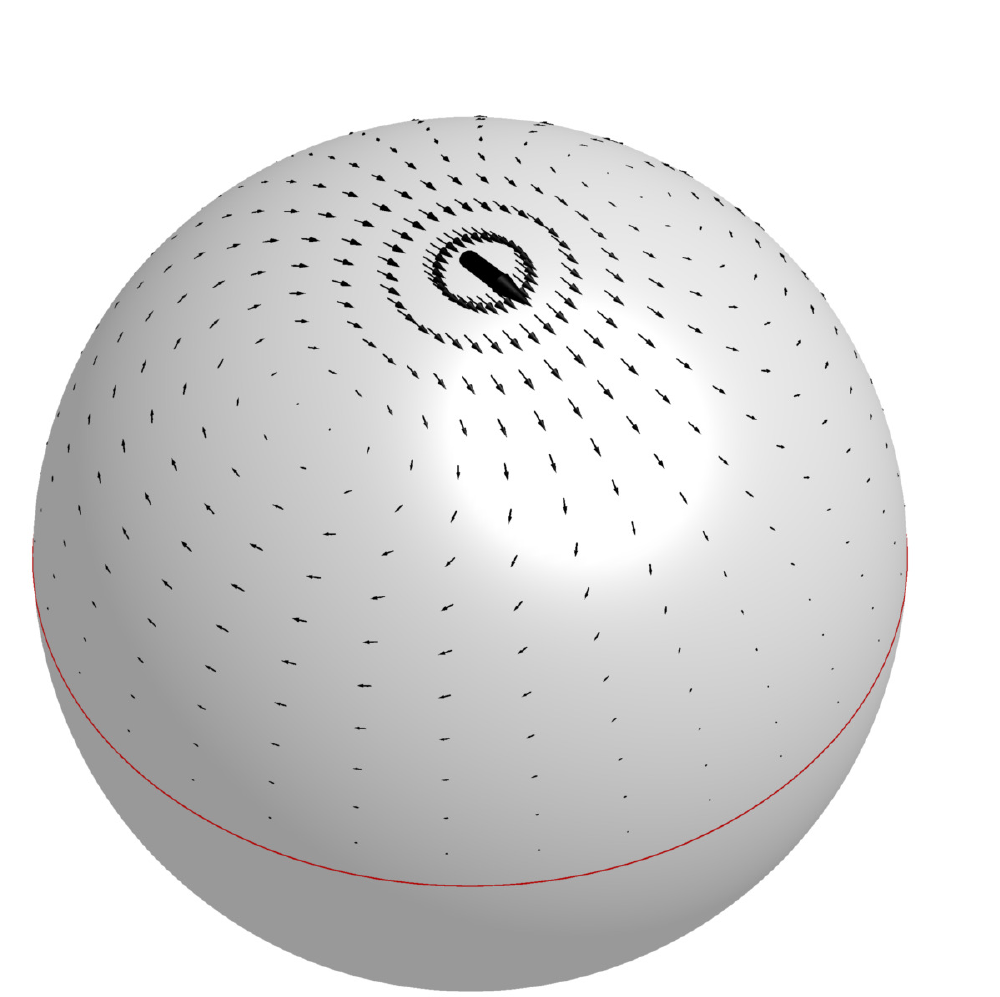}
		\end{minipage}
		\begin{minipage}{0.33\textwidth}
			\vspace{-2.0em}
			\includegraphics[width=\textwidth]{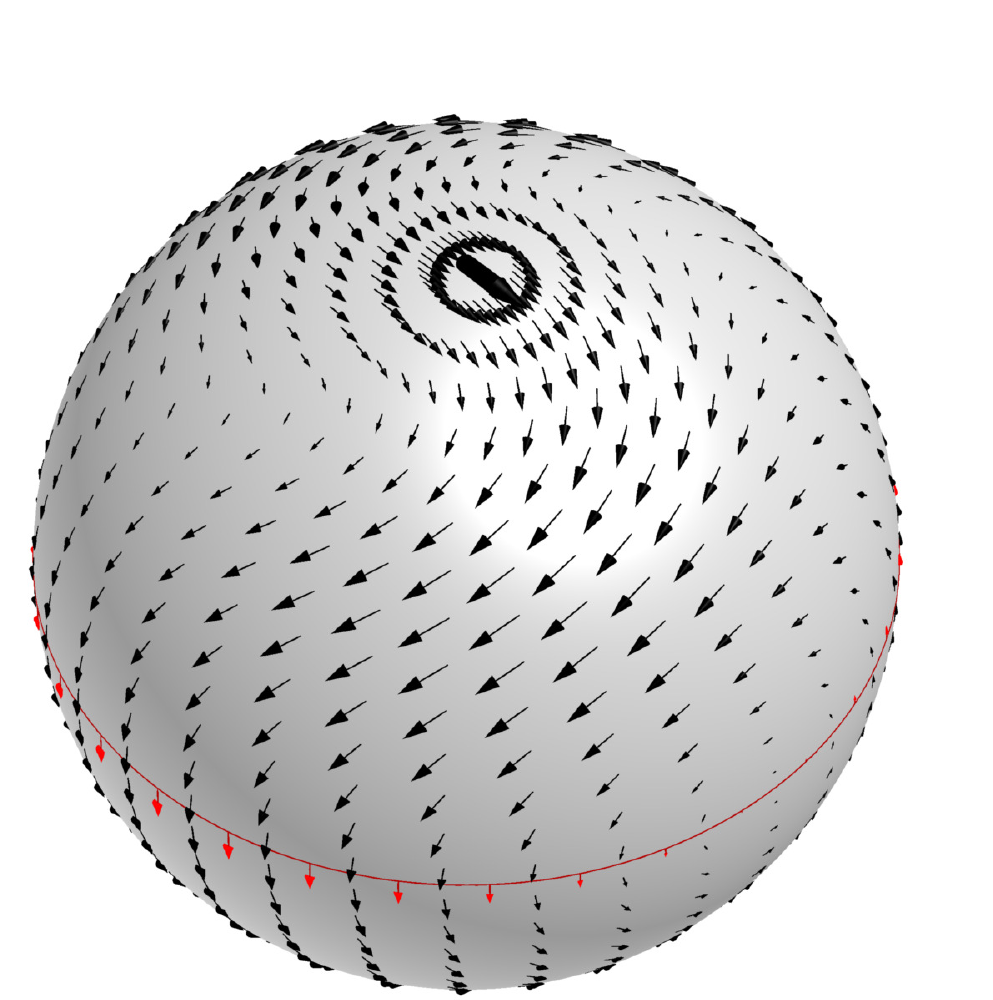}
		\end{minipage}
	}
	\vspace{-1.5em}
	\caption{
		On the left, the posterior mean of conditioning the prior in Example~\ref{example_sphere2} at two opposite points on the equator with tangent vectors pointing north.
		Without sources and sinks, the tangent vectors flow south away from the data.
		In the middle, the posterior mean from Example~\ref{example_sphere_boundary} of a divergence free tangent field on the sphere which is zero at the equator (red) and conditioned at a \emph{single} observation at the north pole. Notice the flow parallel to the equator in middle latitudes, orthogonal to the observation, avoids sinks or sources.
		%\vspace{\parskip}\\
		%\hspace{\parindent}
		On the right, the posterior mean from Example~\ref{example_sphere_boundary_inhomogeneous} of a divergence free tangent field on the sphere with the given boundary condition (red) at the equator being conditioned at a single observation at the north pole.
		Data is displayed artificially bigger.
	}
	\label{figure_example_sphere}
	\label{figure_example_sphere3}
	\label{figure_example_sphere_boundary_inhomogeneous}
\end{figure*}

\begin{example}\label{example_sphere_boundary_inhomogeneous}
	
	Consider smooth divergence free fields on the 2-sphere $X=S^2$, i.e., $f\in\F^{3\times 1}$ with 
	\begin{align*}
	Af=\begin{bmatrix}x&y&z\\\partial_x&\partial_y&\partial_z\end{bmatrix}f=0
	\end{align*}
	and inhomogeneous boundary condition $f_3(x,y,0)=y$.% for all $x,y$ with $x^2+y^2=1$.
	
	The function $\mu=\begin{bmatrix}0&-z&y\end{bmatrix}^T$ is a particular solution.
	Hence, we take it as mean function.
	The matrix $B_1C_1=-B_2C_2$ from Example~\ref{example_sphere_boundary}
	%	\begin{align*}
	%		B=
	%		\begin{bmatrix}
	%		-z^3\partial_y+yz^2\partial_z+2yz\\
	%		z^3\partial_x-xz^2\partial_z-2xz \\
	%		-yz^2\partial_x+xz^2\partial_y\end{bmatrix}
	%	\end{align*}
	parametrizes functions with the corresponding homogeneous boundary condition $f_3(x,y,0)=0$ of functions vanishing at the equator.
	
	Hence, assuming mean zero and squared exponential covariance $k_\F$,
	%($k=\exp\left(-\frac12((x_1-x_2^2)+(y_1-y_2)^2+(z_1-z_2)^2)\right)$)
	the Gaussian process
	%\begin{align*}
	$\GP\left( \mu, (B_1C_1)k_\F((B_1C_1)')^T \right)$
	%\end{align*}
	is a prior distribution in the solutions of the equations and boundary conditions by Lemma~\ref{lemma_gp_operator}, which we demonstrate in Figure~\ref{figure_example_sphere_boundary_inhomogeneous}.
%		for
%		\begin{align*}
%		(B_1C_1)k_\F((B_1C_1)')^T=\\
%		\begin{bmatrix}-z_2^3z_1y_1^2+2y_1z_1^2y_2z_2^2-y_2^2z_1^3z_2-2y_1^2z_2^2+5z_2z_1y_1y_2-2y_2^2z_1^2+z_1^2z_2^2+4y_1y_2&
%		z_2^3z_1y_1x_1-x_1z_1^2y_2z_2^2-y_1z_1^2x_2z_2^2+x_2y_2z_1^3z_2+2x_1y_1z_2^2-5z_2z_1y_1x_2+2x_2y_2z_1^2-4x_2y_1&
%		-z_2^2y_2z_1y_1x_1+z_2z_1^2x_1y_2^2+z_2^2x_2z_1y_1^2-x_2y_1z_1^2y_2z_2-2z_2y_2y_1x_1+2z_2x_2y_1^2-z_1^2z_2x_2\\
%		z_2^3z_1y_1x_1-x_1z_1^2y_2z_2^2-y_1z_1^2x_2z_2^2+x_2y_2z_1^3z_2+2x_1y_1z_2^2-5z_2z_1x_1y_2+2x_2y_2z_1^2-4x_1y_2
%		-z_2^3z_1x_1^2+2x_1z_1^2x_2z_2^2-x_2^2z_1^3z_2-2x_1^2z_2^2+5z_2z_1x_1x_2-2x_2^2z_1^2+z_1^2z_2^2+4x_1x_2
%		z_2^2y_2z_1x_1^2-z_2^2x_2z_1y_1x_1-x_2x_1z_1^2y_2z_2+z_2z_1^2y_1x_2^2+2z_2y_2x_1^2-2z_2x_2y_1x_1-z_1^2z_2y_2\\
%		-z_2^2y_2z_1y_1x_1+z_2z_1^2x_1y_2^2+z_2^2x_2z_1y_1^2-x_2y_1z_1^2y_2z_2+2z_1x_1y_2^2-z_1z_2^2x_1-2z_1x_2y_1y_2&
%		z_2^2y_2z_1x_1^2-z_2^2x_2z_1y_1x_1-x_2x_1z_1^2y_2z_2+z_2z_1^2y_1x_2^2-2z_1x_2x_1y_2+2z_1y_1x_2^2-z_1z_2^2y_1&
%		-z_1z_2x_1^2y_2^2+2z_1z_2x_2x_1y_1y_2-z_1z_2y_1^2x_2^2+z_2z_1x_1x_2+z_2z_1y_1y_2
%		\end{bmatrix}
%		\end{align*}
	
\end{example}	

\begin{figure}
\centering
\centerline{
	\def\len{0.025}
	\def\heada{1.5}
	\def\headb{1}
	\begin{tikzpicture}[scale=5]
	\draw[gray,dotted] (0,0) -- (1,0) -- (1,1) -- (0,1) -- (0,0);
	\foreach \x in {0.0 ,0.05,...,1.001} {
		\foreach \y in {0.0 ,0.05,...,1.001} {
			\pgfmathsetmacro{\vx}{
				1.+(16*(1.+((-3.5)+(1.+(2.5-1.*\y)*\y)*\y)*\y)*exp(-.5*(\x-1)*\x-.5*(\y-1)*\y-.25)+(16*((-2.)+(5.+(.5+((-3.5)+1.*\y)*\y)*\y)*\y)*exp(-.5*(\x-1)*\x-.5*(\y-1)*\y-.25)+16*(1.+((-1.5)+((-1.5)+1.*\y)*\y)*\y)*exp(-.5*(\x-1)*\x-.5*(\y-1)*\y-.25)*\x)*\x)*\x
			}
			\pgfmathsetmacro{\vy}{
				-16*(1.+((-2.)+1.*\y)*\y+((-3.5)+(5.-1.5*\y)*\y+(1.+(.5-1.5*\y)*\y+(2.5+((-3.5)+1.*\y)*\y+(1.*\y-1.)*\x)*\x)*\x)*\x)*exp(-.5*(\x-1)*\x-.5*(\y-1)*\y-.25)*\y
			}
			\draw[-{Latex[length=\heada,width=\headb]}]  (\x,\y) -- (\x+\len*\vx, \y+\len*\vy);
		}
	}
	\draw[-{Latex[length=2*\heada,width=2*\headb]},red,line width=2] (0.5,0.5) -- (0.5,0.5+\len*2);
	\end{tikzpicture}
}
\label{figure_example_plane}
\vspace{-0.5em}
\caption{
	A plot of the model from Example~\ref{example_plane}, conditioned on the vector $(0,1)$ at the point $(0.5,0.5)$, which is plotted artificially bigger and red.
}
\end{figure}
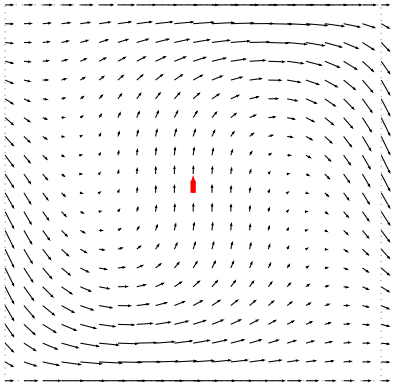

\begin{example}\label{example_plane}
	Consider smooth divergence free fields on the square $X=[0,1]\times[0,1]$ such that no flow in or out of $X$ is possible at the lower and upper boundary of $X$ and there is a constant flow of strength $1$ in $x$-direction at the left and right boundary.
	The divergence-freeness is modelled by the right kernel
	\begin{align*}
		B_1=\begin{bmatrix} \partial_y \\ -\partial_x \end{bmatrix}
	\end{align*}
	of $A=\begin{bmatrix} \partial_x & \partial_y \end{bmatrix}$.
	We model the conditions on the flow by the constant mean function
	\begin{align*}
		\mu:(x,y)\mapsto\begin{bmatrix} 1 \\ 0 \end{bmatrix}
	\end{align*}
	describing flow in $x$-direction and the boundary condition parametrized by
	\begin{align*}
		B_2=\begin{bmatrix}x(x-1) & 0 \\ 0 & y(y-1) \end{bmatrix}.
	\end{align*}
	The nullspace of $\begin{bmatrix}B_1 & B_2\end{bmatrix}$ is
	\begin{align*}
		C:=
		\begin{bmatrix}
		C_1\\C_{2,1}\\C_{2,2}
		\end{bmatrix}
		=		
		\begin{bmatrix}
			x^2y^2-x^2y-xy^2+xy\\
			-y^2\partial_y+y\partial_y-2y+1\\
			x^2\partial_x-x\partial_x+2x-1
		\end{bmatrix}.
	\end{align*}
	and leads to the parametrization
	\begin{align*}
		&P:=B_1C_1=
		-B_2C_2\\
		=&
		\begin{bmatrix}
			x(x-1)(-1+y^2\partial_y+y(-\partial_y+2))\\
			-y(y-1)(-1+x^2\partial_x+x(-\partial_x+2))
		\end{bmatrix}.
	\end{align*}
	Hence, assuming a squared exponential covariance $k_\F$ for the parametrizing function 	the Gaussian process
	\begin{align*}
		\GP\left( \mu, Pk_\F P'^T \right)
	\end{align*}
	is a prior of smooth divergence free fields on $X$ with the above flow conditions.
	We demonstrate this prior in Figure~\ref{figure_example_plane}.
\end{example}

\section{Conclusion}

This paper incorporates prior knowledge into machine learning frameworks.
It presents a novel framework to 
\begin{enumerate}
	\item describe parametrizations for boundary conditions, 
	\item combine parametrizations by intersecting their images, and
	\item build Gaussian process priors with realizations in the solution set of a system of linear differential equations with boundary conditions,
\end{enumerate}	
without any assumptions or approximations.
These priors have been demonstrated on geometric problems and lead to reasonable models with one or two (cf.\ Figure~\ref{figure_example_sphere}) data points.
%The next steps are to
%\begin{enumerate}
%	\item[a.] reduce the need of data in applications like Bayesian optimization, active learning, or design of experiments,
%	\item[b.] allow for stronger extrapolation in robotics,
%	\item[c.] estimate parameters in differential equations via maximum likelihood,
%	%	\item[d.] build controllers incorporating prior knowledge, and
%	\item[d.] construct neural networks, which approximate these Gaussian process priors, reverting the current trend of approximating neural networks via Gaussian processes, cf.\ \citep{neal1996priors,DNNasGP,GarrigaAlonso2018DeepCN,arora2019exact} and references therein.
%\end{enumerate}	

The author thanks the reviewers  for their constructive feedback and by is interested in further work on encoding physical or system-theoretic properties in Gaussian process priors.

\onecolumn

\appendix

%\section{Proof of Lemma~\ref{lemma_pushforward_gaussian}}\label{section_proof_lemma_pushforward_gaussian}
\section{Proof of Lemma~2.2}

Before giving the proof of 
%Lemma~\ref{lemma_pushforward_gaussian}, 
Lemma~2.2, 
we recall the definition (if it exists) of the $\ell$-th cumulant function $\kappa_\ell(g)$
\begin{align*}
%&
\kappa_\ell(g)\left(x^{(1)},\ldots,x^{(\ell)}\right)
%\\
=
%&
\sum_{\pi\in \operatorname{part}(\ell)}(-1)^{|\pi|-1}(|\pi|-1)!
\prod_{\tau\in\pi}E\left(\prod_{i\in \tau}g\left(x^{(i)}\right)\right)
\end{align*}
of a stochastic process $g$, where $\operatorname{part}(\ell)$ is the set of partitions of $\ell$ and $|\pi|$ denotes the cardinality of $\pi$.
In particular, the first two cumulant functions $\kappa_1$ resp.\ $\kappa_2$ are equal to the mean resp.\ covariance function.
Furthermore, $g$ is Gaussian iff all but the first two cumulant functions vanish.

The stochastic process $B_*g$ exists, as $\F$ is an $R$-module and the realizations of $g$ are all contained in $\F$.
The compatibility with expectations proves the following formula for the cumulant functions of $\kappa(B_*g)$ of $B_*g$, where $B^{(i)}$ denotes the operation of $B$ on functions with argument $x^{(i)}\in\R^d$:
\begin{align*}
&\kappa_\ell(B_*g)\left(x^{(1)},\ldots,x^{(\ell)}\right)\\
&=\sum_{\pi\in \operatorname{part}(\ell)}(-1)^{|\pi|-1}(|\pi|-1)!\cdot
\prod_{\tau\in\pi}E\left(\prod_{i\in \tau}(B_*g)\left(x^{(i)}\right)\right)\\
&=\sum_{\pi\in \operatorname{part}(\ell)}(-1)^{|\pi|-1}(|\pi|-1)!\cdot
%\\
%&\hspace{4em}
\prod_{\tau\in\pi}\left(\prod_{i\in \tau}B^{(i)}\right)E\left(\prod_{i\in \tau}g\left(x^{(i)}\right)\right)\\
&\hspace{4em}\mbox{(as $B$ commutes with expectation)}\\
&=\sum_{\pi\in \operatorname{part}(\ell)}(-1)^{|\pi|-1}(|\pi|-1)!\cdot
\;\widehat{B}\;\prod_{\tau\in\pi}E\left(\prod_{i\in \tau}g\left(x^{(i)}\right)\right)\\
&\hspace{4em}\mbox{(as $\pi$ is a partition; $\widehat{B}:=\prod_{i\in \{1,\ldots,\ell\}}B^{(i)}$)}\\
&=\widehat{B}\sum_{\pi\in \operatorname{part}(\ell)}(-1)^{|\pi|-1}(|\pi|-1)!\cdot
\prod_{\tau\in\pi}E\left(\prod_{i\in \tau}g\left(x^{(i)}\right)\right)\\
&\hspace{4em}\mbox{(as $B$ is linear)}\\
&=\widehat{B}\,\kappa_\ell(g)\left(x^{(1)},\ldots,x^{(\ell)}\right)
\end{align*}
As $g$ is Gaussian, the higher ($\ell\ge3$) cumulants $\kappa_\ell(g)$ vanish, hence the higher ($\ell\ge3$) cumulants $\kappa_\ell(B_*g)$ vanish, which implies that $B_*g$ is Gaussian.
The formulas for the mean function resp.\ covariance function follow from the above computation for $\ell=1$ resp.\ $\ell=2$.
\qed
%{\hfill\BlackBox\\[2mm]}

%\section{Proof of Theorem~\ref{theorem_combining_parametrizations}}\label{section_proof_theorem_combining_parametrizations}
\section{Proof of Theorem~5.2}

Before giving the proof of 
%Theorem~\ref{theorem_combining_parametrizations},
Theorem~5.2,
we recall some definitions and facts from homological algebra and category theory \cite{nlab,MLCWM,weihom,ce}.
A collection of two morphisms with the same source $A_2\xleftarrow{\alpha_1} B\xrightarrow{\alpha_2} A_2$ is  a \emph{span} and a collection of two morphisms with the same range $C_2\xrightarrow{\gamma_1} D\xleftarrow{\gamma_2} C_2$ is a \emph{cospan}.
Given a cospan $C_2\xrightarrow{\gamma_1} D\xleftarrow{\gamma_2} C_2$, an object $P$ together with two morphisms $\delta_1:P\to C_1$ and $\delta_2:P\to C_2$ is called a \emph{pullback}, if $\gamma_1\circ\delta_1=\gamma_2\circ\delta_2$ and for every $P'$ with two morphisms $\delta_1':P'\to C_1$ and $\delta_2':P'\to C_2$ such that $\gamma_1\circ\delta_1'=\gamma_2\circ\delta_2'$ there exists a unique morphism $\pi:P'\to P$ such that $\delta_1\circ\pi=\delta_1'$ and $\delta_2\circ\pi=\delta_2'$.
Pullbacks are the generalization of intersections.
Given a span $A_2\xleftarrow{\alpha_1} B\xrightarrow{\alpha_2} A_2$, an object $P$ together with two morphisms $\beta_1:A_1\to P$ and $\beta_2:A_2\to P$ is called a \emph{pushout}, if $\beta_1\circ\alpha_1=\beta_2\circ\alpha_2$ and for every $P'$ with two morphisms $\beta_1':A_1\to P'$ and $\beta_2':A_2\to P'$ such that  $\beta_1'\circ\alpha_1=\beta_2'\circ\alpha_2$ there exists a unique morphism $\pi:P\to P'$ such that $\beta_1\circ\pi=\beta_1'$ and $\beta_2\circ\pi=\beta_2'$.
Pullbacks and Pushouts exist in the category of finitely presented modules.
Given an $R$-module $M$, an epimorphism $M\xleftarrowdbl{}\R^m$ is a \emph{free cover} of $M$ and a monomorphism $M\hookrightarrow\R^m$ is a \emph{free hull} of $M$.
Every finitely presented $R$-module has a free cover, but only a free hull iff it corresponds to a controllable system.
Given an $R$-module $M$, the \emph{contravariant hom-functor} $\hom_R(-,M)$ is the hom-set $\hom_R(A,M)=\{\psi:A\to M\mid \psi\ R\text{-module homomorphism}\}$ when applied to an $R$-module $A$ and application to an $R$-module homomorphism $\varphi:A\to B$ gives $\hom_R(\varphi,M): \hom_R(B,M)\to \hom_R(A,M): \beta\mapsto \beta\circ\varphi$.
If $R$ is a commutative, then $\hom_R(-,M)$ is a functor to the category of $R$-modules, otherwise it is a functor to the category of Abelian groups.

By 
%Corollary~\ref{corollary_parametrizable},
Corollary~3.7,
the assumptions of 
%Theorem~\ref{theorem_combining_parametrizations}
Theorem~5.2
ensure that we have a parametrization $C$ of the system defined by $B$.
As $C$ is the nullspace of $B$, we have $B_1C_1=-B_2C_2$.

The parametrization of an intersection of parametrizations  $B_1\F^{\ell_1''}\cap B_2\F^{\ell_2''}$ is given by the image of the pullback $P$ of the cospan $\F^{\ell_1''}\xrightarrow{B_1}\F^{\ell}\xleftarrow{B_2}\F^{\ell_2''}$ in $\F^{\ell}$ by \citep[15.10.8.a]{eis}.
The approach of
%Theorem~\ref{theorem_combining_parametrizations}
Theorem~5.2
computes a subset\footnote{To get the full image, we need $\F$ to be an injective module.} of this image via a free cover $P\xleftarrowdbl{}\F^m$ of this pullback $P$ as image of $B_1C_1=-B_2C_2$, as depicted in the following commutative diagram:
\begin{center}
	\begin{tikzpicture}[node distance=0.5cm]]
	\node (im) {$\F^{\ell}$};
	\node[above right = of im] (s1)  {$\F^{\ell_1''}$};
	\node[below right = of im] (s2)  {$\F^{\ell_2''}$};
	\node[above right = of s2] (p)  {$P$};
	\node[right = of p] (c)  {$\F^m$};
	\draw[->] (s1) -- node[above left]{$B_1$} (im);
	\draw[->] (s2) -- node[below left]{$B_2$} (im);
	\draw[->] (p) -- (s1);
	\draw[->] (p) -- (s2);
	\draw[->>] (c) -- (p);
	\draw[->] (c) -- node[above right]{$C_1$} (s1);
	\draw[->] (c) -- node[below right]{$-C_2$} (s2);
	\end{tikzpicture}
\end{center}
As in 
%Theorem~\ref{theorem_parametrizable}
Theorem~3.6
and
%Corollary~\ref{corollary_parametrizable},
Corollary~3.7,
the computation is done dually over the ring $R$.
There, the cospan $R^{1\times\ell_1''}\xrightarrow{C_1}R^{1\times m}\xleftarrow{C_2}R^{1\times\ell_2''}$ defines a free hull $Q\lhook\joinrel\xrightarrow{C}R^{1\times m}$ of the pushout $Q$ of the span $R^{1\times\ell_1''}\xleftarrow{B_1}R^{1\times\ell}\xrightarrow{B_2}R^{1\times\ell_2''}$.
Then applying the dualizing hom-functor $\hom_R(-,\F)$ transforms this to the function space $\F$.
\qed

Even though all operations in this proof are algorithmic \citep{BL}, 
%Theorem~\ref{theorem_combining_parametrizations}
Theorem~5.2
describes a computationally more efficient algorithm.

\section{Code}\label{appendix_code}

The following computation have been performed in Maple with the OreModules package \citep{CQR07}.

\begin{example}[General Code for GP regression]
	\phantom{test}

	\begin{maplegroup}
		\begin{mapleinput}
			\mapleinline{active}{1d}{# code for GP regression
				GP:=proc(Kf,
				\ points,yy,epsilon)
				local n,m,kf,K,s1,s2,alpha,KStar;
				\ n:=nops(points);
				\ m:=RowDimension(Kf);
				\ s1:=map(
				\ \ a->[x1=a[1],y1=a[2],z1=a[3]],
				\ \ points);
				\ s2:=map(
				\ \ a->[x2=a[1],y2=a[2],z2=a[3]],
				\ \ points);
				\ kf:=convert(Kf,listlist);
			}{}
		\end{mapleinput}
		\begin{mapleinput}
			\mapleinline{active}{1d}{
				\ K:=convert(
				\ \ evalf(
				\ \ \ map(
				\ \ \ \ a->map(
				\ \ \ \ \ b->convert(
				\ \ \ \ \ \ subs(a,subs(b,kf)),
				\ \ \ \ \ \ Matrix),
				\ \ \ \ \ s2),
				\ \ \ \ s1)),
				\ \ Matrix):
				\ alpha:=yy.(K+epsilon\symbol{94}2)\symbol{94}(-1);
				\ KStar:=map(
				\ \ a->subs(a,kf),
				\ \ s1):
				\ KStar:=subs(
				\ \ [x2=x,y2=y,z2=z],KStar):
				\ KStar:=convert(
				\ \ map(op,KStar),Matrix):
				\ return alpha.KStar;
				end:
			}{}
		\end{mapleinput}
	\end{maplegroup}
	
\end{example}

\begin{example}[Code for Example~\makeatletter \ifcsdef{r@example_sphere2}{\ref{example_sphere2}}{3.9}\makeatother]
	%globe_divergence_free
\phantom{test}

	\begin{maplegroup}
		\begin{mapleinput}
			\mapleinline{active}{1d}{restart;
			}{}
		\end{mapleinput}
	\end{maplegroup}
	\begin{maplegroup}
		\begin{mapleinput}
			\mapleinline{active}{1d}{with(OreModules):
			}{}
		\end{mapleinput}
	\end{maplegroup}
	\begin{maplegroup}
		\begin{mapleinput}
			\mapleinline{active}{1d}{with(LinearAlgebra):
			}{}
		\end{mapleinput}
	\end{maplegroup}
	\begin{maplegroup}
		\begin{mapleinput}
			\mapleinline{active}{1d}{Alg:=DefineOreAlgebra(diff=[Dx,x],
				\ diff=[Dy,y], diff=[Dz,z],
				\ diff=[Dx1,x1], diff=[Dy1,y1],
				\ diff=[Dz1,z1], diff=[Dx2,x2],
				\ diff=[Dy2,y2], diff=[Dz2,z2],
				\ polynom=[x,y,z,x1,x2,y1,y2,z1,z2]):}{}
		\end{mapleinput}
	\end{maplegroup}
	\begin{maplegroup}
		\begin{mapleinput}
			\mapleinline{active}{1d}{A:=<<x,Dx>|<y,Dy>|<z,Dz>>;
			}{}
		\end{mapleinput}
	 \[A:=\left[ \begin {array}{ccc} x&y&z\\ \noalign{\medskip}{\it Dx}&{\it Dy
	}&{\it Dz}\end {array} \right] \]
	\end{maplegroup}
	\begin{maplegroup}
		\begin{mapleinput}
			\mapleinline{active}{1d}{# combine
				B:=Involution(
				\ SyzygyModule(
				\ \ Involution(A,Alg),
				\ \ Alg),
				\ Alg);
			}{}
		\end{mapleinput}
		\[B:=\left[ \begin {array}{c} z{\it Dy}-{\it Dz}\,y\\ \noalign{\medskip}-{
			\it Dx}\,z+{\it Dz}\,x\\ \noalign{\medskip}{\it Dx}\,y-{\it Dy}\,x
		\end {array} \right]\]
	\end{maplegroup}
	\begin{maplegroup}
		\begin{mapleinput}
			\mapleinline{active}{1d}{# check parametrization
				A1:=SyzygyModule(B,Alg):
				ReduceMatrix(A,A1,Alg);
				ReduceMatrix(A1,A,Alg);
			}{}
		\end{mapleinput}
		\[[]\]
		\[[]\]
	\end{maplegroup}
	\begin{maplegroup}
		\begin{mapleinput}
			\mapleinline{active}{1d}{# covariance for
				# parametrizing function
				SE:=exp(-1/2*(x1-x2)\symbol{94}2
				\ -1/2*(y1-y2)\symbol{94}2-1/2*(z1-z2)\symbol{94}2):
			}{}
		\end{mapleinput}
	\end{maplegroup}
	\begin{maplegroup}
		\begin{mapleinput}
			\mapleinline{active}{1d}{
				Kg:=unapply(
				\ DiagonalMatrix([SE]),
				\ (x1,y1,z1,x2,y2,z2)):
			}{}
		\end{mapleinput}
	\end{maplegroup}
	\begin{maplegroup}
		\begin{mapleinput}
			\mapleinline{active}{1d}{
				# prepare covariance
				P2:=ApplyMatrix(B,
				\ [xi(x,y,z)], Alg):
				P2:=convert(P2,list):
			}{}
		\end{mapleinput}
	\end{maplegroup}
	\smallskip
	\begin{maplegroup}
		\begin{mapleinput}
			\mapleinline{active}{1d}{
				l1:=[x=x1,y=y1,z=z1,
				\ Dx=Dx1,Dy=Dy1,Dz=Dz1]:
			}{}
		\end{mapleinput}
	\end{maplegroup}
	\begin{maplegroup}
		\begin{mapleinput}
			\mapleinline{active}{1d}{
				l2:=[x=x2,y=y2,z=z2,
				\ Dx=Dx2,Dy=Dy2,Dz=Dz2]:
			}{}
		\end{mapleinput}
	\end{maplegroup}
	\smallskip
	\begin{maplegroup}
		\begin{mapleinput}
			\mapleinline{active}{1d}{# construct covariance
				# apply from one side
				Kf:=convert(
				\ map(
				\ \ b->subs(
				\ \ \ [xi(x1,y1,z1)=b[1]],
				\ \ \ subs(l1,P2)),
				\ \ convert(
				\ \ \ Kg(x1,y1,z1,x2,y2,z2),
				\ \ \ listlist)),
				\ Matrix):
				# apply from other side
				Kf:=convert(
				\ expand(
				\ \ map(
				\ \ \ b->subs(
				\ \ \ \ [xi(x2,y2,z2)=b[1]],
				\ \ \ \ subs(l2,P2)),
				\ \ \ convert(
				\ \ \ \ Transpose(Kf),
				\ \ \ \ listlist))),
				\ Matrix):
			}{}
		\end{mapleinput}
	\end{maplegroup}
	\smallskip
	\begin{maplegroup}
		\begin{mapleinput}
			\mapleinline{active}{1d}{
				gp:=unapply(
				\ evalf(convert(
				\ \ GP(Kf,
				\ \ \ [[1,0,0],[-1,0,0]],
				\ \ \ <<0>|<0>|<1>|<0>|<0>|<1>>,
				\ \ \ 1e-5),
				\ \ list)),
				\ (x,y,z)):
			}{}
		\end{mapleinput}
	\end{maplegroup}
	\begin{maplegroup}
		\begin{mapleinput}
			\mapleinline{active}{1d}{gp(x,y,z):
				factor(simplify(%));
			}{}
		\end{mapleinput}
		\begin{align*}
			&0.7015\,\\
			&[ 
			\,z \left(-\,{{\rm e}^{x- 0.5\,x^2- 0.5\,y^2- 0.5\,z^2}}
			+{{\rm e}^{-x- 0.5\,x^2- 0.5\,y^2- 0.5\,z^2}} \right),\\
			&yz \left({{\rm e}^{x- 0.5\,x^2- 0.5\,y^2- 0.5\,z^2}}
			+{{\rm e}^{- 1.0\,x- 0.5\,x^2- 0.5\,y^2- 0.5\,z^2}} \right),\\
			&-y^2{{\rm e}^{x- 0.5\,x^2- 0.5\,y^2- 0.5\,z^2}}
			+x{{\rm e}^{x- 0.5\,x^2-0.5\,y^2- 0.5\,z^2}}\\
			&-y^2{{\rm e}^{-x- 0.5\,x^2-0.5\,y^2- 0.5\,z^2}}
			-x{{\rm e}^{-x- 0.5\,x^2- 0.5\,y^2- 0.5\,z^2}}
			]
		\end{align*}
	\end{maplegroup}

\end{example}

\begin{example}[Code for Example~Code for Example~\makeatletter \ifcsdef{r@example_ode_boundary}{\ref{example_ode_boundary}}{4.1}\makeatother]
\phantom{test}

% worksheets/ODE_boundary.mw

\begin{maplegroup}
	\begin{mapleinput}
		\mapleinline{active}{1d}{restart;with(LinearAlgebra):
		}{}
	\end{mapleinput}
\end{maplegroup}
\begin{maplegroup}
	\begin{mapleinput}
		\mapleinline{active}{1d}{k:=(x,y)->exp(-1/2*(x-y)\symbol{94}2);
		}{}
	\end{mapleinput}
	\mapleresult
	\[k:=( {x,y} )\mapsto {{\rm e}^{-1/2\, \left( x-y \right) ^{2}}}\]
\end{maplegroup}
\begin{maplegroup}
	\begin{mapleinput}
		\mapleinline{active}{1d}{K:=<
			\ <k(0,0),subs(x=0,diff(k(x,0),x)),
			\ \ k(1,0),subs(x=1,diff(k(x,0),x))>|
			\ <subs(y=0,diff(k(0,y),y)),
			\ \ subs([x=0,y=0],diff(k(x,y),x,y)),
			\ \ subs(y=0,diff(k(1,y),y)),
			\ \ subs([x=1,y=0],diff(k(x,y),x,y))>|
			\ <k(0,1),subs(x=0,diff(k(x,1),x)),
			\ \ k(1,1),subs(x=1,diff(k(x,1),x))>|
			\ <subs(y=1,diff(k(0,y),y)),
			\ \ subs([x=0,y=1],diff(k(x,y),x,y)),
			\ \ subs(y=1,diff(k(1,y),y)),
			\ \ subs([x=1,y=1],diff(k(x,y),x,y))>
			>:
		}{}
	\end{mapleinput}
	\begin{mapleinput}
		\mapleinline{active}{1d}{
			K:=simplify(K);
		}{}
	\end{mapleinput}
	\[ K:=\left[ \begin {array}{cccc} 1&0&{{\rm e}^{-1/2}}&-{{\rm e}^{-1/2}}
	\\ \noalign{\medskip}0&1&{{\rm e}^{-1/2}}&0\\ \noalign{\medskip}{
		{\rm e}^{-1/2}}&{{\rm e}^{-1/2}}&1&0\\ \noalign{\medskip}-{{\rm e}^{-1
			/2}}&0&0&1\end {array} \right] 
	\]
\end{maplegroup}
\begin{maplegroup}
	\begin{mapleinput}
		\mapleinline{active}{1d}{# posterior covariance
		}{}
	\end{mapleinput}
	\smallskip
	\begin{mapleinput}
		\mapleinline{active}{1d}{
			K_star:=unapply(
			<<k(x,0)>|
			\ <subs(y=0,diff(k(x,y),y))>|
			\ <k(x,1)>|
			\ <subs(y=1,diff(k(x,y),y))>>,x):
		}{}
	\end{mapleinput}
	\begin{mapleinput}
		\mapleinline{active}{1d}{
			K_inv:=simplify(K\symbol{94}(-1)):
		}{}
	\end{mapleinput}
	\begin{mapleinput}
	\mapleinline{active}{1d}{
		d:=denom(K_inv[1,1]):
	}{}
	\end{mapleinput}
	\begin{mapleinput}
	\mapleinline{active}{1d}{
		K_inv_d:=simplify(d*K_inv):
	}{}
	\end{mapleinput}
	\smallskip
	\begin{mapleinput}
	\mapleinline{active}{1d}{
		1/d*simplify(
		\ (<<d*k(x,y)>>
		\ \ -K_star(x).K_inv_d.
		\ \ Transpose(K_star(y)))[1,1]
		);
	}{}
	\end{mapleinput}
	\begin{align*}
	&{\rm e}^{-\frac12(x-y)^2}
	-{\frac{{\rm e}^{-\frac12x^2-\frac12y^2}}{{{\rm e}^{-2}}-3\,{{\rm e}^{-1}}+1}}\cdot\\
	&\Big(
	(xy-x-y+2)	{{\rm e}^{x+y-1}}
	+(xy+1)\\
	& 	+(-2xy+x+y-1) \left({{\rm e}^{x+y-2}}+{{\rm e}^{-1}}\right)\\
	& 	+(xy-y+1)		{{\rm e}^{y-2}}
	+(xy-x+1)		{{\rm e}^{x-2}}\\
	& 	+(y-x-2) 		{{\rm e}^{y-1}}
	+(x-y-2) 		{{\rm e}^{x-1}}	
	\Big),
	\end{align*}
\end{maplegroup}
\begin{maplegroup}
	\begin{mapleinput}
		\mapleinline{active}{1d}{}{}
	\end{mapleinput}
\end{maplegroup}

\end{example}

\begin{example}[Code for Example~\makeatletter \ifcsdef{r@example_sphere2b}{\ref{example_sphere2b}}{5.3}\makeatother]
	
	% globe_as_intersection
	
	\phantom{test}
	
	\begin{maplegroup}
		\begin{mapleinput}
			\mapleinline{active}{1d}{
				restart;
			}{}
		\end{mapleinput}
		\begin{mapleinput}
		\mapleinline{active}{1d}{
			with(OreModules):
		}{}
		\end{mapleinput}
		\begin{mapleinput}
		\mapleinline{active}{1d}{
			with(LinearAlgebra):
		}{}
		\end{mapleinput}
	\end{maplegroup}
	\smallskip
	\begin{maplegroup}
		\begin{mapleinput}
		\mapleinline{active}{1d}{
			Alg:=DefineOreAlgebra(
			\ diff=[Dx,x], diff=[Dy,y],
			\ diff=[Dz,z], polynom=[x,y,z]):
		}{}
		\end{mapleinput}
	\end{maplegroup}
	\begin{maplegroup}
		\begin{mapleinput}
			\mapleinline{active}{1d}{
				A1:=<<Dx>|<Dy>|<Dz>>;
			}{}
		\end{mapleinput}
	\[
	A\!\mathit{1}:= \left[ \begin {array}{ccc} {\it Dx}&{\it Dy}&{\it Dz}\end {array}
	\right] 
	\]
		\begin{mapleinput}
			\mapleinline{active}{1d}{
				B1:=Involution(
				\ SyzygyModule(
				\ \ Involution(A1,Alg),
				\ \ Alg),
				\ Alg):
				# reorder columns
				B1:=B1.<<0,0,-1>|<1,0,0>|<0,-1,0>>;
			}{}
		\end{mapleinput}
	\[
	B\!\mathit{1}:=\left[ \begin {array}{ccc} {\it Dz}&{\it Dy}&0\\ \noalign{\medskip}0&
	-{\it Dx}&{\it Dz}\\ \noalign{\medskip}-{\it Dx}&0&-{\it Dy}
	\end {array} \right] \]
	
	\end{maplegroup}
	\begin{maplegroup}
		\begin{mapleinput}
			\mapleinline{active}{1d}{
				A2:=<<x>|<y>|<z>>;
			}{}
		\end{mapleinput}
		\[A\!\mathit{2}:=\left[ \begin {array}{ccc} x&y&z\end {array} \right] \]
		\begin{mapleinput}
			\mapleinline{active}{1d}{
				B2:=Involution(
				\ SyzygyModule(
				\ \ Involution(A2,Alg),
				\ \ Alg),
				\ Alg):
				# reorder columns
				B2:=B2.<<0,0,1>|<-1,0,0>|<0,1,0>>;
			}{}
		\end{mapleinput}
		\[B\!\mathit{2}:=\left[ \begin {array}{ccc} 0&z&-y\\ \noalign{\medskip}-z&0&x
		\\ \noalign{\medskip}y&-x&0\end {array} \right] \]
		
	\end{maplegroup}
	\begin{maplegroup}
		\begin{mapleinput}
			\mapleinline{active}{1d}{{\small MinimalParametrizations(<B1|B2>,Alg):}
				C:=%[2]:
				# Normalize Columns
				C:=C.DiagonalMatrix([-1,-1,1]);
			}{}
		\end{mapleinput}
		\[C:=\left[ \begin {array}{ccc} x&{\it Dx}&0\\ \noalign{\medskip}y&{\it Dy
		}&0\\ \noalign{\medskip}z&{\it Dz}&0\\ \noalign{\medskip}{\it Dx}&0&x
		\\ \noalign{\medskip}{\it Dy}&0&y\\ \noalign{\medskip}{\it Dz}&0&z
		\end {array} \right] 
		\]
	\end{maplegroup}
	\begin{maplegroup}
		\begin{mapleinput}
			\mapleinline{active}{1d}{#check, parametrization:
				BB:=SyzygyModule(C,Alg):
				ReduceMatrix(BB,<B1|B2>,Alg);
				ReduceMatrix(<B1|B2>,BB,Alg);
			}{}
		\end{mapleinput}
		\[[]\]
		\[[]\]
	\end{maplegroup}
	\begin{maplegroup}
		\begin{mapleinput}
			\mapleinline{active}{1d}{
				# B1*C1
				BB1:=Mult(B1,C[1..3,1..3],Alg);
			}{}
		\end{mapleinput}
	\[BB\!\mathit{1}:=\left[ \begin {array}{ccc} -z{\it Dy}+{\it Dz}\,y&0&0
	\\ \noalign{\medskip}{\it Dx}\,z-{\it Dz}\,x&0&0\\ \noalign{\medskip}-
	{\it Dx}\,y+{\it Dy}\,x&0&0\end {array} \right] 
	\]
		\begin{mapleinput}
		\mapleinline{active}{1d}{ 
				# -B2*C2
				BB2:=-Mult(B2,C[4..6,1..3],Alg);
				}{}
		\end{mapleinput}
	\[BB\!\mathit{2}:=\left[ \begin {array}{ccc} -z{\it Dy}+{\it Dz}\,y&0&0
	\\ \noalign{\medskip}{\it Dx}\,z-{\it Dz}\,x&0&0\\ \noalign{\medskip}-
	{\it Dx}\,y+{\it Dy}\,x&0&0\end {array} \right] 
	\]
	\end{maplegroup}
	\begin{maplegroup}
		\begin{mapleinput}
			\mapleinline{active}{1d}{#For comparison:
				B_old:=<<y*Dz-z*Dy,
				\ -x*Dz+z*Dx,-y*Dx+x*Dy>>;
				
			}{}
		\end{mapleinput}
	\[ B\_old:=\left[ \begin {array}{c} -z{\it Dy}+{\it Dz}\,y\\ \noalign{\medskip}{
		\it Dx}\,z-{\it Dz}\,x\\ \noalign{\medskip}-{\it Dx}\,y+{\it Dy}\,x
	\end {array} \right] 
	\]
	\end{maplegroup}

\end{example}

\begin{example}[Code for Example~\makeatletter \ifcsdef{r@example_sphere_boundary}{\ref{example_sphere_boundary}}{5.4}\makeatother]
	% globe_boundary2
	\phantom{test}

	\begin{maplegroup}
		\begin{mapleinput}
			\mapleinline{active}{1d}{restart;
			}{}
		\end{mapleinput}
	\end{maplegroup}
	\begin{maplegroup}
		\begin{mapleinput}
			\mapleinline{active}{1d}{with(Janet):
			}{}
		\end{mapleinput}
	\end{maplegroup}
	\begin{maplegroup}
		\begin{mapleinput}
			\mapleinline{active}{1d}{with(OreModules):
			}{}
		\end{mapleinput}
	\end{maplegroup}
	\begin{maplegroup}
		\begin{mapleinput}
			\mapleinline{active}{1d}{with(LinearAlgebra):
			}{}
		\end{mapleinput}
	\end{maplegroup}
	\begin{maplegroup}
		\begin{mapleinput}
			\mapleinline{active}{1d}{with(plots):
			}{}
		\end{mapleinput}
	\end{maplegroup}
	\smallskip
	\begin{maplegroup}
		\begin{mapleinput}
			\mapleinline{active}{1d}{Alg:=DefineOreAlgebra(diff=[Dx,x],
				\ diff=[Dy,y], diff=[Dz,z],
				\ diff=[Dx1,x1], diff=[Dy1,y1],
				\ diff=[Dz1,z1], diff=[Dx2,x2],
				\ diff=[Dy2,y2], diff=[Dz2,z2],
				\ polynom=[x,y,z,x1,x2,y1,y2,z1,z2]):
			}{}
		\end{mapleinput}
	\end{maplegroup}
	\smallskip
	\begin{maplegroup}
		\begin{mapleinput}
			\mapleinline{active}{1d}{
				# div-free fields on S\symbol{94}2
			}{}
		\end{mapleinput}
		\begin{mapleinput}
			\mapleinline{active}{1d}{
				B1:=<<y*Dz-z*Dy,-x*Dz+z*Dx,-y*Dx+x*Dy>>:
			}{}
		\end{mapleinput}
	\end{maplegroup}
	\smallskip
	\begin{maplegroup}
		\begin{mapleinput}
			\mapleinline{active}{1d}{
				# parametrize equator=0
				B2:=DiagonalMatrix([z\$3]);
			}{}
		\end{mapleinput}
		\[B\!\mathit{2}:=\left[ \begin {array}{ccc} z&0&0\\ \noalign{\medskip}0&z&0\\ \noalign{\medskip}0&0&z\end {array} \right] \]
	\end{maplegroup}
	\begin{maplegroup}
		\begin{mapleinput}
			\mapleinline{active}{1d}{# combine
				B:=<B1|B2>:
				C:=Involution(
				\ SyzygyModule(
				\ \ Involution(B,Alg),
				\ \ Alg),
				\ Alg);
			}{}
		\end{mapleinput}
		\[C:=\left[ \begin {array}{c} {z}^{2}\\ \noalign{\medskip}{\it Dy}\,{z}^{2
			}-{\it Dz}\,yz-2\,y\\ \noalign{\medskip}-{\it Dx}\,{z}^{2}+x{\it Dz}\,
			z+2\,x\\ \noalign{\medskip}{\it Dx}\,yz-{\it Dy}\,xz\end {array}
			\right] 
			\]
	\end{maplegroup}
	\begin{maplegroup}
		\begin{mapleinput}
				\mapleinline{active}{1d}{
					# check parametrization
					BB:=SyzygyModule(C,Alg):
					ReduceMatrix(B,BB,Alg);
					}{}
		\end{mapleinput}
		\[[]\]
		\begin{mapleinput}
			\mapleinline{active}{1d}{
				ReduceMatrix(BB,B,Alg);
				# new relation!}{}
		\end{mapleinput}
		\[\left[ \begin {array}{cccc} 0&x&y&z\end {array} \right]\]
	\end{maplegroup}
	\begin{maplegroup}
		\begin{mapleinput}
			\mapleinline{active}{1d}{
				# the new parametrization
				P:=Mult(B1,[[C[1,1]]],Alg);
			}{}
		\end{mapleinput}
		\[P:= \left[ \begin {array}{c} z \left( -{\it Dy}\,{z}^{2}+{\it Dz}\,yz+2\,
		y \right) \\ \noalign{\medskip}z \left( {\it Dx}\,{z}^{2}-x{\it Dz}\,z
		-2\,x \right) \\ \noalign{\medskip} \left( -{\it Dx}\,y+{\it Dy}\,x
		\right) {z}^{2}\end {array} \right]
		\]
	\end{maplegroup}
	\begin{maplegroup}
		\begin{mapleinput}
			\mapleinline{active}{1d}{# sanity check = P
				-Mult(B2,C[2..4,1..1],Alg);}{}
		\end{mapleinput}
		\[\left[ \begin {array}{c} z \left( -{\it Dy}\,{z}^{2}+{\it Dz}\,yz+2\,y \right) \\ \noalign{\medskip}z \left( {\it Dx}\,{z}^{2}-x{\it Dz}\,z-2\,x \right) \\ \noalign{\medskip} \left( -{\it Dx}\,y+{\it Dy}\,x\right) {z}^{2}\end {array} \right]
			\]
	\end{maplegroup}
	\begin{maplegroup}
		\begin{mapleinput}
			\mapleinline{active}{1d}{
				# covariance for
				# parametrizing function
				SE:=exp(-1/2*(x1-x2)\symbol{94}2
				\ -1/2*(y1-y2)\symbol{94}2-1/2*(z1-z2)\symbol{94}2):
			}{}
		\end{mapleinput}
		\begin{mapleinput}
				\mapleinline{active}{1d}{
					Kg:=unapply(
					\ DiagonalMatrix([SE]),
					\ (x1,y1,z1,x2,y2,z2)):
				}{}
		\end{mapleinput}
	\end{maplegroup}
	\begin{maplegroup}
		\begin{mapleinput}
			\mapleinline{active}{1d}{
				# prepare covariance
				P2:=ApplyMatrix(P,
				\ [xi(x,y,z)], Alg):
				P2:=convert(P2,list):
			}{}
		\end{mapleinput}
	\end{maplegroup}
	\begin{maplegroup}
		\begin{mapleinput}
			\mapleinline{active}{1d}{
				l1:=[x=x1,y=y1,z=z1,
				\ Dx=Dx1,Dy=Dy1,Dz=Dz1]:
			}{}
		\end{mapleinput}
	\end{maplegroup}
	\begin{maplegroup}
		\begin{mapleinput}
			\mapleinline{active}{1d}{
				l2:=[x=x2,y=y2,z=z2,
				\ Dx=Dx2,Dy=Dy2,Dz=Dz2]:
			}{}
		\end{mapleinput}
	\end{maplegroup}
	\begin{maplegroup}
		\begin{mapleinput}
			\mapleinline{active}{1d}{
				# construct covariance
				# apply from one side
				Kf:=convert(
				\ map(
				\ \ b->subs(
				\ \ \ [xi(x1,y1,z1)=b[1]],
				\ \ \ subs(l1,P2)),
				\ \ convert(
				\ \ \ Kg(x1,y1,z1,x2,y2,z2),
				\ \ \ listlist)),
				\ Matrix):
			}{}
		\end{mapleinput}
	\end{maplegroup}
	\begin{maplegroup}
		\begin{mapleinput}
			\mapleinline{active}{1d}{
				# apply from other side
				Kf:=convert(
				\ expand(
				\ \ map(
				\ \ \ b->subs(
				\ \ \ \ [xi(x2,y2,z2)=b[1]],
				\ \ \ \ subs(l2,P2)),
				\ \ \ convert(
				\ \ \ \ Transpose(Kf),
				\ \ \ \ listlist))),
				\ Matrix):
			}{}
		\end{mapleinput}
	\end{maplegroup}
	\begin{maplegroup}
		\begin{mapleinput}
			\mapleinline{active}{1d}{
				gp:=unapply(
				\ piecewise(z<0,[0,0,0],
				\ evalf(convert(
				\ \ GP(Kf,[[0,0,1]],<1|0|0>,1e-5),
				\ \ list))),
				\ (x,y,z)):
			}{}
		\end{mapleinput}
	\end{maplegroup}
	\begin{maplegroup}
		\begin{mapleinput}
			\mapleinline{active}{1d}{
				gp(x,y,z) assuming z>0:
				factor(simplify(%));
			}{}
		\end{mapleinput}
		\begin{align*}
			[
			- 0.6065\,z(-z^2+zy^2+ 2y^2)\,{{\rm e}^{- 0.5\,{x}^{2}- 0.5\,{y}^{2}+z-0.5\,{z}^{2}}},\\
			0.6065\,xyz(z+2)\,{{\rm e}^{- 0.5\,{x}^{2}- 0.5\,{y}^{2}+z- 0.5\,{z}^{2}}} ,\\
			- 0.6065\,xz^2\,{{\rm e}^{- 0.5\,{x}^{2}- 0.5\,{y}^{2}+z- 0.5\,{z}^{2}}}
			]
		\end{align*}
	\end{maplegroup}
	
\end{example}

\begin{example}[Code for Example~\makeatletter \ifcsdef{r@example_sphere_boundary_inhomogeneous}{\ref{example_sphere_boundary_inhomogeneous}}{6.1}\makeatother]
	\phantom{test}
	% Inhomogeneous Equation Boundary Condition

	\begin{maplegroup}
		\begin{mapleinput}
			\mapleinline{active}{1d}{restart;
			}{}
		\end{mapleinput}
	\end{maplegroup}
	\begin{maplegroup}
		\begin{mapleinput}
			\mapleinline{active}{1d}{with(OreModules):
			}{}
		\end{mapleinput}
	\end{maplegroup}
	\begin{maplegroup}
		\begin{mapleinput}
			\mapleinline{active}{1d}{with(LinearAlgebra):
			}{}
		\end{mapleinput}
	\end{maplegroup}
	\begin{maplegroup}
		\begin{mapleinput}
			\mapleinline{active}{1d}{Alg:=DefineOreAlgebra(diff=[Dx,x],
				\ diff=[Dy,y], diff=[Dz,z],
				\ diff=[Dx1,x1], diff=[Dy1,y1],
				\ diff=[Dz1,z1], diff=[Dx2,x2],
				\ diff=[Dy2,y2], diff=[Dz2,z2],
				\ polynom=[x,y,z,x1,x2,y1,y2,z1,z2]):}{}
		\end{mapleinput}
	\end{maplegroup}
	\begin{maplegroup}
		\begin{mapleinput}
			\mapleinline{active}{1d}{B1:=<<y*Dz-z*Dy,-x*Dz+z*Dx,-y*Dx+x*Dy>>;}{}
		\end{mapleinput}
		\[B\!\mathit{1}:=\left[ \begin {array}{c} -z{\it Dy}+{\it Dz}\,y\\ \noalign{\medskip}z
		{\it Dx}-{\it Dz}\,x\\ \noalign{\medskip}-{\it Dx}\,y+{\it Dy}\,x
		\end {array} \right] 
		\]
	\end{maplegroup}
	\begin{maplegroup}
		\begin{mapleinput}
			\mapleinline{active}{1d}{mu:=<0,-z,y>;
			}{}
		\end{mapleinput}
		\[\mu:=\left[ \begin {array}{c} 0\\ \noalign{\medskip}-z
		\\ \noalign{\medskip}y\end {array} \right]
		\]
	\end{maplegroup}
	\begin{maplegroup}
		\begin{mapleinput}
			\mapleinline{active}{1d}{
				#check:
				A1:=Matrix(1,3,[[Dx,Dy,Dz]]):
				A2:=Matrix(1,3,[[x,y,z]]):
				ApplyMatrix(A1,mu,Alg);
				ApplyMatrix(A2,mu,Alg);
			}{}
		\end{mapleinput}
		\[[0]\]
		\[[0]\]
	\end{maplegroup}
	\begin{maplegroup}
		\begin{mapleinput}
			\mapleinline{active}{1d}{# the new parametrization
				P:=Mult(B1,[[z\symbol{94}2]],Alg);}{}
		\end{mapleinput}
		\[
			P:= \left[ \begin {array}{c} z \left( -{\it Dy}\,{z}^{2}+{\it Dz}\,yz+2\,y \right) \\ \noalign{\medskip}z \left( {\it Dx}\,{z}^{2}-x{\it Dz}\,z-2\,x \right) \\ \noalign{\medskip} \left( -{\it Dx}\,y+{\it Dy}\,x\right) {z}^{2}\end {array} \right]
		\]
	\end{maplegroup}
	\begin{maplegroup}
		\begin{mapleinput}
			\mapleinline{active}{1d}{# covariance for
				# parametrizing function
				SE:=exp(-1/2*(x1-x2)\symbol{94}2
				\ -1/2*(y1-y2)\symbol{94}2-1/2*(z1-z2)\symbol{94}2):
			}{}
		\end{mapleinput}
	\end{maplegroup}
	\smallskip
	\begin{maplegroup}
		\begin{mapleinput}
			\mapleinline{active}{1d}{
				Kg:=unapply(
				\ DiagonalMatrix([SE]),
				\ (x1,y1,z1,x2,y2,z2)):
			}{}
		\end{mapleinput}
	\end{maplegroup}
	\begin{maplegroup}
		\begin{mapleinput}
			\mapleinline{active}{1d}{
				# prepare covariance
				P2:=ApplyMatrix(P,
				\ [xi(x,y,z)], Alg):
				P2:=convert(P2,list):
			}{}
		\end{mapleinput}
	\end{maplegroup}
	\smallskip
	\begin{maplegroup}
		\begin{mapleinput}
			\mapleinline{active}{1d}{
				l1:=[x=x1,y=y1,z=z1,
				\ Dx=Dx1,Dy=Dy1,Dz=Dz1]:
			}{}
		\end{mapleinput}
	\end{maplegroup}
	\smallskip
	\begin{maplegroup}
		\begin{mapleinput}
			\mapleinline{active}{1d}{
				l2:=[x=x2,y=y2,z=z2,
				\ Dx=Dx2,Dy=Dy2,Dz=Dz2]:
			}{}
		\end{mapleinput}
	\end{maplegroup}
	\smallskip
	\begin{maplegroup}
		\begin{mapleinput}
			\mapleinline{active}{1d}{
				# construct covariance
				# apply from one side
				Kf:=convert(
				\ map(
				\ \ b->subs(
				\ \ \ [xi(x1,y1,z1)=b[1]],
				\ \ \ subs(l1,P2)),
				\ \ convert(
				\ \ \ Kg(x1,y1,z1,x2,y2,z2),
				\ \ \ listlist)),
				\ Matrix):
			}{}
		\end{mapleinput}
	\end{maplegroup}
	\begin{maplegroup}
		\begin{mapleinput}
			\mapleinline{active}{1d}{
				# apply from other side
				Kf:=convert(
				\ expand(
				\ \ map(
				\ \ \ b->subs(
				\ \ \ \ [xi(x2,y2,z2)=b[1]],
				\ \ \ \ subs(l2,P2)),
				\ \ \ convert(
				\ \ \ \ Transpose(Kf),
				\ \ \ \ listlist))),
				\ Matrix):
			}{}
		\end{mapleinput}
	\end{maplegroup}
	\begin{maplegroup}
		\begin{mapleinput}
			\mapleinline{active}{1d}{
				p:=[0,0,1]:
			}{}
		\end{mapleinput}
	\end{maplegroup}
	\smallskip
	\begin{maplegroup}
		\begin{mapleinput}
			\mapleinline{active}{1d}{
				mu_p:=Transpose(
				\ subs(
				\ \ [x=p[1],y=p[2],z=p[3]],
				\ \ mu)):
			}{}
		\end{mapleinput}
	\end{maplegroup}
	\smallskip
	\begin{maplegroup}
		\begin{mapleinput}
			\mapleinline{active}{1d}{
				gp:=unapply(
				\ factor(simplify(
				\ \ convert(
				\ \ \ GP(Kf,[p],<1|0|0>-mu_p,1e-5),
				\ \ \ list)))
				\ \ +convert(mu,list),
				\ (x,y,z)):
			}{}
		\end{mapleinput}
	\end{maplegroup}
	\begin{maplegroup}
		\begin{mapleinput}
			\mapleinline{active}{1d}{
				gp(x,y,z);
			}{}
		\end{mapleinput}
	
	\begin{align*}
	[
	- 0.6065\,z(-z^2+zy^2+ 2y^2)\,{{\rm e}^{- 0.5\,{x}^{2}- 0.5\,{y}^{2}+z-0.5\,{z}^{2}}},\\
	0.6065\,xyz(z+2)\,{{\rm e}^{- 0.5\,{x}^{2}- 0.5\,{y}^{2}+z- 0.5\,{z}^{2}}}-z ,\\
	- 0.6065\,xz^2\,{{\rm e}^{- 0.5\,{x}^{2}- 0.5\,{y}^{2}+z- 0.5\,{z}^{2}}}+y
	]
	\end{align*}
	\end{maplegroup}
			
\end{example}

\begin{example}[Code for Example~\makeatletter \ifcsdef{r@example_sphere_boundary_inhomogeneous}{\ref{example_sphere_boundary_inhomogeneous}}{6.2}\makeatother]
	\phantom{test}	
	
	\begin{maplegroup}
		\begin{mapleinput}
			\mapleinline{active}{1d}{restart;
			}{}
		\end{mapleinput}
	\end{maplegroup}
	\begin{maplegroup}
		\begin{mapleinput}
			\mapleinline{active}{1d}{with(OreModules):
			}{}
		\end{mapleinput}
	\end{maplegroup}
	\begin{maplegroup}
		\begin{mapleinput}
			\mapleinline{active}{1d}{with(LinearAlgebra):
			}{}
		\end{mapleinput}
	\end{maplegroup}
	\begin{maplegroup}
		\begin{mapleinput}
			\mapleinline{active}{1d}{Alg:=DefineOreAlgebra( 
			\ diff=[Dx,x], diff=[Dy,y],
			\ diff=[Dx1,x1], diff=[Dy1,y1],
			\ diff=[Dx2,x2], diff=[Dy2,y2],
			\ polynom=[x,y,x1,x2,y1,y2]):
			}{}
		\end{mapleinput}
	\end{maplegroup}
	\begin{maplegroup}
		\begin{mapleinput}
			\mapleinline{active}{1d}{A:=<<Dx>|<Dy>>;
			}{}
		\end{mapleinput}
		\mapleresult
		\begin{maplelatex}
			\[
			\left[ \begin {array}{cc} {\it Dx}&{\it Dy}\end {array} \right]
			\]
		\end{maplelatex}
	\end{maplegroup}
	\begin{maplegroup}
		\begin{mapleinput}
			\mapleinline{active}{1d}{B1:=Involution(
				\ SyzygyModule(
				\ \ Involution(A,Alg),
				\ \ Alg),
				\ Alg);
			}{}
		\end{mapleinput}
		\mapleresult
		\begin{maplelatex}
			\[
			\left[ \begin {array}{c} {\it Dy}\\ \noalign{\medskip}-{\it Dx}
			\end {array} \right] 
			\]
		\end{maplelatex}
	\end{maplegroup}
	\begin{maplegroup}
		\begin{mapleinput}
			\mapleinline{active}{1d}{mu:=<1,0>;
			}{}
		\end{mapleinput}
		\mapleresult
		\begin{maplelatex}
			\[
			\left[ \begin {array}{c} 1\\ \noalign{\medskip}0\end {array} \right]
			\]
		\end{maplelatex}
	\end{maplegroup}
	\begin{maplegroup}
		\begin{mapleinput}
			\mapleinline{active}{1d}{B2:=<<(x-1)*x,0>|<0,(y-1)*y>>;
			}{}
		\end{mapleinput}
		\mapleresult
		\begin{maplelatex}
			\[
			\left[ \begin {array}{cc}  \left( x-1 \right) x&0
			\\ \noalign{\medskip}0& \left( y-1 \right) y\end {array} \right] 
			\]
		\end{maplelatex}
	\end{maplegroup}
	\begin{maplegroup}
		\begin{mapleinput}
			\mapleinline{active}{1d}{# combine
				B:=<B1|B2>:
				C:=Involution(
				\ SyzygyModule(
				\ \ Involution(B,Alg),
				\ \ Alg),
				\ Alg);}{}
		\end{mapleinput}
		\mapleresult
		\begin{maplelatex}
			\[
			 \left[ \begin {array}{c} {x}^{2}{y}^{2}-{x}^{2}y-x{y}^{2}+xy
			\\ \noalign{\medskip}-{\it Dy}\,{y}^{2}+{\it Dy}\,y-2\,y+1
			\\ \noalign{\medskip}{\it Dx}\,{x}^{2}-{\it Dx}\,x+2\,x-1\end {array}
			\right] 
			\]
		\end{maplelatex}
	\end{maplegroup}
	\begin{maplegroup}
		\begin{mapleinput}
			\mapleinline{active}{1d}{# the new parametrization
				P:=Mult(B1,C[1,1],Alg);
			}{}
		\end{mapleinput}
		\mapleresult
		\begin{maplelatex}
			\[
			\left[ \begin {array}{c} x \left( -1+{\it Dy}\,{y}^{2}+ \left( -{\it 
				Dy}+2 \right) y \right)  \left( x-1 \right) \\ \noalign{\medskip}-
			\left( y-1 \right) y \left( -1+{\it Dx}\,{x}^{2}+ \left( -{\it Dx}+2
			\right) x \right) \end {array} \right] 
			\]
		\end{maplelatex}
	\end{maplegroup}
	\begin{maplegroup}
		\begin{mapleinput}
			\mapleinline{active}{1d}{# covariance for
				# parametrizing function
				SE:=exp(-1/2*(x1-x2)\symbol{94}2
				-1/2*(y1-y2)\symbol{94}2):
			}{}
		\end{mapleinput}
	\end{maplegroup}
	\begin{maplegroup}
		\begin{mapleinput}
			\mapleinline{active}{1d}{Kg:=unapply(
				DiagonalMatrix([SE]),
				(x1,y1,x2,y2)):
			}{}
		\end{mapleinput}
	\end{maplegroup}
	\begin{maplegroup}
		\begin{mapleinput}
			\mapleinline{active}{1d}{# prepare covariance
				P2:=ApplyMatrix(P,
				[xi(x,y)], Alg):
				P2:=convert(P2,list):
			}{}
		\end{mapleinput}
	\end{maplegroup}
	\begin{maplegroup}
		\begin{mapleinput}
			\mapleinline{active}{1d}{l1:=[x=x1,y=y1,
				Dx=Dx1,Dy=Dy1]:
			}{}
		\end{mapleinput}
	\end{maplegroup}
	\begin{maplegroup}
		\begin{mapleinput}
			\mapleinline{active}{1d}{l2:=[x=x2,y=y2,
				Dx=Dx2,Dy=Dy2]:
			}{}
		\end{mapleinput}
	\end{maplegroup}
	\begin{maplegroup}
		\begin{mapleinput}
			\mapleinline{active}{1d}{# construct covariance
				# apply from one side
				Kf:=convert(
				\ map(
				\ \ b->subs(
				\ \ \ [xi(x1,y1)=b[1]],
				\ \ \ subs(l1,P2)),
				\ \ convert(
				\ \ \ Kg(x1,y1,x2,y2),
				\ \ \ listlist)),
				\ Matrix):
				# apply from other side
				Kf:=convert(
				\ expand(
				\ \ map(
				\ \ \ b->subs(
				\ \ \ \ [xi(x2,y2)=b[1]],
				\ \ \ \ subs(l2,P2)),
				\ \ \ convert(
				\ \ \ \ Transpose(Kf),
				\ \ \ \ listlist))),
				\ Matrix):
			}{}
		\end{mapleinput}
	\end{maplegroup}
	\begin{maplegroup}
		\begin{mapleinput}
			\mapleinline{active}{1d}{# code for GP regression
				GP:=proc(Kf,
				\ points,yy,epsilon)
				local n,m,kf,K,s1,s2,alpha,KStar;
				\ n:=nops(points);
				\ m:=RowDimension(Kf);
				\ s1:=map(
				\ \ a->[x1=a[1],y1=a[2]],
				\ \ points);
				\ s2:=map(
				\ \ a->[x2=a[1],y2=a[2]],
				\ \ points);
				\ kf:=convert(Kf,listlist);
				\ K:=convert(
				\ \ evalf(
				\ \ \ map(
				\ \ \ \ a->map(
				\ \ \ \ \ b->convert(
				\ \ \ \ \ \ subs(a,subs(b,kf)),
				\ \ \ \ \ \ Matrix),
				\ \ \ \ \ s2),
				\ \ \ s1)),
				\ \ Matrix):
				\ alpha:=yy.(K+epsilon\symbol{94}2)\symbol{94}(-1);
				\ KStar:=map(
				\ \ a->subs(a,kf),
				\ \ s1):
				\ KStar:=subs(
				\ \ [x2=x,y2=y],KStar):
				\ KStar:=convert(
				\ \ map(op,KStar),Matrix):
				\ return alpha.KStar;
				end:
			}{}
		\end{mapleinput}
	\end{maplegroup}
	\begin{maplegroup}
		\begin{mapleinput}
			\mapleinline{active}{1d}{p:=[1/2,1/2]:
				mu_p:=Transpose(
				subs(
				\ [x=p[1],y=p[2]],
				\ mu)):
				gp:=unapply(
				\ factor(simplify(
				\ \ convert(
				\ \ \ GP(Kf,[p],<0|1>-mu_p,1e-5),
				\ \ \ list)))
				\ +convert(mu,list),
				\ (x,y));
			}{}
		\end{mapleinput}
		\mapleresult
		\begin{maplelatex}
			\begin{align*}
			( {x,y} )\mapsto &{{\rm e}^{- 0.25+ 0.5x- 0.5x^2+0.5y- 0.5y^2}}\cdot\\
			&[1+ 16x \left(  y^4(x-1)+ y^3(x-2.5)(x-1)+ y^2(0.5x+ 1-1.5x^2)- y(x-1)(x-2.33)+ (x-1)^2 \right) ,\\
		&- 16y \left( x^4(y-1)+ x^3(y-2.5)(y-1)+x^2(0.5y+ 1-1.5y^2)- x(y-1)(y-2.33)+ (y-1)^2 \right) ]
		
			\end{align*}
		\end{maplelatex}
	\end{maplegroup}
	\begin{maplegroup}
		\begin{mapleinput}
			\mapleinline{active}{1d}{}{}
		\end{mapleinput}
	\end{maplegroup}
	\begin{maplegroup}
		\begin{mapleinput}
			\mapleinline{active}{1d}{}{}
		\end{mapleinput}
	\end{maplegroup}
\end{example}

\twocolumn

\bibliographystyle{apalike}
\def\cprime{$'$} \def\cprime{$'$} \def\cprime{$'$} \def\cprime{$'$}
  \def\cprime{$'$}

\end{document}